\documentclass{article}

\usepackage[preprint]{neurips_2021}

\usepackage[utf8]{inputenc} 
\usepackage[T1]{fontenc}    
\usepackage{hyperref}       
\usepackage{url}            
\usepackage{booktabs}       
\usepackage{amsfonts}       
\usepackage{nicefrac}       
\usepackage{microtype}      
\usepackage{xcolor}         
\usepackage{algorithm}
\usepackage{algorithmic}
\usepackage{graphicx}
\usepackage{subfigure}
\usepackage{wrapfig}
\usepackage{booktabs} 
\newlength{\without}

\usepackage{multirow}
\usepackage[thinc]{esdiff}
\usepackage{amssymb}
\usepackage{amsthm}
\newtheorem{theorem}{Theorem}
\newtheorem{lemma}{Lemma}
\newtheorem{definition}{Definition}

\newtheorem{proposition}{Proposition}

\usepackage{mathtools}

\title{Combining Online Learning and Offline Learning for Contextual Bandits with Deficient Support}

\author{%
  Hung Tran-The\thanks{Correspondence to: Hung Tran-The <\texttt{hung.tranthe@deakin.edu.au}>.}, Sunil Gupta, Thanh Nguyen-Tang, Santu Rana, Svetha Venkatesh \\
  Applied Artificial Intelligence Institute\\
  Deakin University, Australia\\
}

\begin{document}

\maketitle

\begin{abstract}
We address policy learning with logged data in contextual bandits. Current offline-policy learning algorithms are mostly based on inverse propensity score (IPS) weighting requiring the logging policy to have \emph{full support} i.e. a non-zero probability for any context/action of the evaluation policy. However, many real-world systems do not guarantee such logging policies, especially when the action space is large and many actions have poor or missing rewards. With such \emph{support deficiency}, the offline learning fails to find optimal policies. We propose a novel approach that uses a hybrid of  offline learning with online exploration. The online exploration is used to explore unsupported actions in the logged data whilst offline learning is used to exploit supported actions from the logged data avoiding unnecessary explorations. Our approach determines an optimal policy with theoretical guarantees using the minimal number of online explorations. We demonstrate our algorithms' effectiveness empirically on a diverse collection of datasets.
\end{abstract}
\section{Introduction}
Many interactive systems (e.g., recommender systems, search engines) can be modeled as contextual bandit problems where the learner repeatedly observes a context (e.g., user profile, query), takes an action (e.g., recommended product) and observes a reward (e.g., purchase, click) for the chosen action. The goal of such interactive systems is to collect as much reward as possible.

When logged data is available and acquiring new data is expensive, offline contextual bandit methods that leverage offline data without further online exploration are an ideal approach. Several solutions follow this approach  (e.g., \citep{Dud11,Strehl10,Swaminathan17,Wang17,kato2020}). However, most current solutions including the regression-based direct modeling (DM), the inverse propensity score (IPS) \citep{Horvitz1952}, and the doubly robust (DR) \citep{Dud11} estimators are unsuitable for  real-world applications. The DM approach learns a reward model and estimates the value of an evaluation policy. This estimate is usually biased because of model misspecification, which is difficult to correct without the knowledge of the evaluation policy.  IPS uses importance weighting to correct for the proportions of actions in offline data and is guaranteed to be unbiased \citep{Strehl10,Dud11,farajtabar18a,swaminathan15a,XieLLWZP19}. DR can combine the two estimators and achieve unbiased estimates if only if one of the   estimators is unbiased \citep{Dud11}. Unfortunately, IPS is only unbiased if the logging policy has full support, that is,  all  actions have nonzero probability of being selected under the logging policy. Full support, however, does not exist in most real-world systems, especially when the action space is large and many actions have poor rewards. For example, in a recommender system with a large catalogue, only a small percentage of the items (actions) are in the support of the logging policy. This renders off-policy learning methods that rely on IPS unusable, e.g. current off-policy learning algorithms such as the counterfactual estimator based algorithm \citep{swaminathan15a}, the CAIPWL estimator based algorithm \citep{zhou2018offline} and MLIPS estimator \citep{XieLLWZP19} unusable, as being based on IPS, they require full support.

With such support deficiency, the unbiasedness of both IPS and DM methods breaks down.  \citet{Sachdeva2020} approaches the off-policy contextual-bandit learning by either seeking solutions in a restricted search space close to supported actions, or  learn rewards of unsupported actions from logging policy by regression extrapolation. However, in such a restricted search space, we may only find only suboptimal solutions especially when the optimal actions are far from the supported actions. Also, learning rewards of unsupported actions is inefficient because we have no any information about unsupported actions. In the setting where online exploration is possible (though possibly expensive), a natural approach to resolve the support deficiency is to acquire new data to inform about the unsupported actions in the offline data. This idea is surprisingly unexplored. How to efficiently resolve support deficiency in offline learning with online data and how such a hybrid approach can be more efficient than pure online learning with contextual bandits (e.g., \citep{chu11a,agrawal13,agarwalb14}) remain open questions.

In this paper, we address the deficiency support problem for contextual bandits by combining offline and online learning. The online exploration is used to explore unsupported actions in the logged data whilst offline-policy learning is used to exploit the supported actions from the logged data avoiding unnecessary explorations. Our approach can efficiently find the optimal policy with theoretical guarantees. Our main contributions are as follows:
\begin{itemize}
  \item We study the deficiency support for contextual bandits from a novel hybrid perspective combining offline and online learning;
  \item We introduce two algorithms. The first algorithm leverages reward models learned from offline data to reduce the number of online explorations (Section 4.2). The second algorithm improves the efficiency of the first one to further reduce the number of online explorations by exploiting "good" context-action pairs in offline data i.e., the reward corresponding to this context-action pair is approximately optimal. (Section 4.3);
  \item We provide a unifying analysis for our hybrid approach that generalizes the analysis of either online or offline learning; in particular, we show that both of our proposed algorithms obtain sublinear regrets;
  \item Finally, we demonstrate our algorithms' effectiveness empirically on a diverse collection of data sets in Section 5.
\end{itemize}
\vspace*{-4mm}
\section{Related Works}
The contextual-bandit learning can be viewed as off-policy learning in the reinforcement learning(RL) \citep{levine2020offline} which considers learning optimal policies in the sequential decision-making setting. Similar to the works of \citet{Sachdeva2020} in contextual bandit, \citet{Liu20} proposes safe learning algorithms based on policy iteration and value iteration by restricting the policy space. In linear MDP, \cite{jin2020} show that an offline learning method following the pessimism principle finds the optimal policy in sublinear time in linear MDP when the offline data and the evaluation well explore the action and context space. Though the "well-explored" condition in \cite{jin2020} is weaker than full support condition, the pessimism principle stills fail in the presence of deficient support in disjoint linear models where the parameters are not shared among different actions.

Combining offline learning and online learning has been also studied but for the purposes rather than resolving the support deficiency problem, e.g., in multi-armed bandits \citep{shivaswamy12}, in latent bandits \citep{Zhou16,Hong20}, in confounded bandits \citep{tennenholtz2020}, in causal inference \citep{Li20}, and in online fine-tuning \citep{levine2020offline}. In our knowledge, we are the first to address the support deficiency problem by combining online and offline learning.
\vspace*{-3mm}
\section{Problem Setting}
\subsection{Contextual Bandit}
We consider the contextual bandit problem. Formally, we define by $\mathcal A = \{1, 2, ..., K \}$ the set of actions, and a learner interacts with the environment in discrete iteration $t = 1, 2,...$. In iteration $t$:
\begin{enumerate}
  \item The environment outputs a context $x_t \in \mathcal X \subset \mathbb{R}^d$ sampled from a unknown distribution $P_x$.
  \item Based on observed payoffs in previous trials and the current context $x_t$, the learner chooses an action $a_t \in \mathcal A$, and receives payoff $r_t$. It is important to note here that no feedback information is observed for unchosen actions $a \not = a_t$.
  \item The learner then improves its action-selection strategy using all information  $\{(x_i, a_i, r_i)\}$.
\end{enumerate}
We consider a disjoint linear model of reward in the form $r_{x,a} = \langle x, \theta^*_{a} \rangle + \eta$, where
$\theta^*_{a}$ is \emph{unknown} parameter vector of action $a$, and $\eta$ is a $\sigma$-subgaussian random noise. This model is \emph{disjoint} in the sense that the parameters are not shared among different actions \citep{Lihong2010}.

Given a space $\Pi$ of policies $\pi : \mathcal X \rightarrow \mathcal A$, we define the value of any policy $\pi \in \Pi$ as $R(\pi) = \mathbb{E}_{x \sim P_x} \mathbb{E}_{a \sim \pi(x)}[r_{x,a}]$. The goal of the policy learning algorithm is to find an optimal policy $\pi^*$ that has the maximum value $R(\pi)$.

In this work, we assume that we have an offline dataset $S = \{x^{\mu}_i, a^{\mu}_i, r^{\mu}_i\}_{i =1}^n$ in which the contexts $\{x^{\mu}_i\}_{i=1}^n$ are i.i.d sampled from the distribution $P_x$, the actions $\{a^{\mu}_i\}_{i=1}^n$ are generated by some fixed behavior policy, denoted by $\mu$ which is a mapping from context $x \in \mathcal X$ to a probability over actions;  the corresponding rewards $\{r^{\mu}_i\}_{i=1}^n$ are generated by the same reward model $r^{\mu}_i = \langle x^{\mu}_i, \theta^*_{a^{\mu}_i} \rangle + \eta_i$.
\subsubsection{Performance Measure}
We measure the performance of algorithms based on (1) the number of online explorations given the same set of contexts and (2) the following regret $R_T$.

Assume that the learner is allowed to interact with the environment in $T$ times, the first goal of the learner is to select $T$ contexts $\{x_1, x_2 ,..., x_T\}$ from the environment and take appropriate actions to maximize the cumulative rewards $\sum_{t=1}^{T} \langle x_t, \theta^*_{a_t} \rangle$, where $a_t$ is the chosen action corresponding to context $x_t$ at iteration $t$. It is equivalent to minimize the regret $R_T = \sum_{t=1}^{T}\langle x_t, \theta^*_{\pi^*(x_t)} \rangle -\langle x_t, \theta^*_{a_t} \rangle$, where $\pi^*(x) \in \text{argmax}_{a \in \mathcal A} \langle x, \theta^*_{a} \rangle$.

In the presence of offline data, algorithms do not always need to take actions and call reward functions while interacting with the environment. For example, when the reward model of an action is learned well from offline data by some manner, the learner can predict the received reward when if that action is chosen and hence avoids calling reward functions unnecessarily which may be expensive/risky in many applications (See our Algorithm 2). Thus, the second goal of learner is to reduce efficiently reward callings. We note that we distinguish between the interaction and the exploration. When we say that the learner performs an \emph{online exploration} we mean it perform an action and receive the reward function.

\subsection{Support Deficiency}
Given dataset $S$ from logging policy $\mu$, we are interested in estimating the value of a policy $\pi$. The unbiasedness of an estimator is crucial for policy learning. While this property is hard to be guaranteed by the DM estimator, it can be guaranteed by the IPS estimator, $ \hat{R}_{IPS}(\pi) = \frac{1}{|S|} \sum_{i=1}^{|S|} \frac{\pi(a^{\mu}_i|x^{\mu}_i)}{\mu(a^{\mu}_i|x^{\mu}_i)} r^{\mu}_i$, if full support between the target policy $\pi$ and the logging policy $\mu$ is satisfied:
\begin{definition}[Full Support]
The logging policy $\mu$ is said to have full support for $\pi$ if for all context-action pairs $(x,a) \in \mathcal X \times \mathcal A$, whenever $\pi(a|x) > 0$ we also have $\mu(a|x) > 0$.
\end{definition}
In this case, we have $\mathbb{E}_{S}[\hat{R}_{IPS}] = R(\pi)$. This property is needed for not only IPS, but also for similar estimators like counterfactual estimator \citep{swaminathan15a}, SNIPS \citep{Swaminathan15b}, CAIPWL \citep{zhou2018offline}, MLIPS \citep{XieLLWZP19}, CAB \citep{su19a}. However, this condition may not be feasible in many real-world systems, especially if the action space is large and many actions have poor or missing reward. When the full support requirement is violated, we call it \emph{deficient support}. To quantify how support deficient a logging policy is, we denote the set of \emph{unsupported actions} for context $x$ under a policy $\pi$ as
$$\mathcal  U(x, \pi) = \{a \in \mathcal A| \pi(a|x) = 0 \}.$$
We note that sets of unsupported actions of different contexts may be different. Deficient support happens when $\mathcal U(x, \mu) \cap \mathcal U^c(x, \pi) \not = \emptyset$, where $\mathcal U^c$ is the complement of $\mathcal U$. The larger the intersection set, the more deficient is the support for the context $x$. In the support deficiency, the bias of the IPS is then characterized by the expected reward on the unsupported actions.
\begin{proposition}[Proposition 1 of \citep{Sachdeva2020}]
The bias of an estimator $\hat{R}_{IPS}(\pi)$ for target policy $\pi$ is equal to the expected reward on the unsupported action sets, i.e.,
$$\text{bias}(\hat{R}_{IPS}(\pi)) = \mathbb{E}_{S}[\hat{R}_{IPS}(\pi)] - R(\pi) =   \mathbb{E}_{x}[- \sum_{a \in \mathcal U(x, \mu)} \pi(a|x)\delta(x, a)],$$
where $R(\pi)$ is the true value of $\pi$ which is defined as $R(\pi) = \mathbb{E}_{x \in \mathcal X} \mathbb{E}_{a \in \pi(a|x)}[r(x,a)]$, and $\delta(x, a) = \mathbb{E}_{S}[r(x,a)|x, a]$.
\end{proposition}
In fact, it is not possible to eliminate the bias only with the offline data when the set $\mathcal U(x, \mu) \cap \mathcal U^c(x, \pi)$ is large except one special case where all the reward models are learnt. We mention this in Remark 2.
\vspace*{-4mm}
\section{Policy Learning Algorithms}
In this section, we provide algorithms for learning the optimal policy in presence of support deficiency. All proofs
are provided in the Supplementary Material.
\vspace*{-3mm}
\subsection{mOFUL Algorithm}
The simplest idea is to design an online learning where there is no exploitation of offline data. We adapt the OFUL algorithm \citep{Abbasi11} to our setting where each action $a$ has an unknown particular parameter $\theta_a$ instead only one parameter across all actions $\theta^*$ as in the original paper. For each action $a$, the algorithm maintains a confidence interval $C_{t,a} \subseteq \mathbb{R}^d$ at each iteration $t$ for its parameter $\theta^*_a$ such that this interval contains $\theta^*_a$ with high probability. Assume that $(x_1, a_1, r_1), ..., (x_t, a_t, r_t)$ are context-action-reward triples generated by the algorithm up to iteration $t$. Given an action $a$, we denote by $i_1,.., i_{t_a}$ iterations up to $t$ in which $a$ is chosen by the algorithm. Let $\hat{\theta}_{t,a}$ be the least-squares estimate of $\theta^*_a$ with regularization parameter $\lambda > 0$:
\begin{eqnarray}
\hat{\theta}_{t,a} = (X^T_{t,a}X_{t,a} + \lambda I)^{-1} X^T_{t,a}Y_{t,a} \label{eq_10}
\end{eqnarray}
where matrix  $X_{t,a}$  rows are $x^T_{i_1}, ..., x^T_{i_t}$, $Y_{t,a} = (r_{i_1}, ..., r_{i_t})^T$; $x^T$ and $X^T$ denotes the transpose of a vector $x$ and matrix $X$ respectively.
Based on $\hat{\theta}_{t,a}$, the interval $C_{t,a}$ is defined as follows:
\begin{eqnarray}
C_{t,a}= \{\theta \in \mathbb{R}^d: ||\hat{\theta}_{t,a} - \theta||_{\overline{V}_{t,a}} \le \beta_t \}, \label{eq_11}
\end{eqnarray}
where $\overline{V}_{t,a} = \lambda I + \sum_{i=1}^{t} x_ix_i^T \mathbf{1}(a_i = a)$ and
$\sqrt{\beta_t(\delta)} = \sigma \sqrt{d\text{log}(\frac{K-(1 + tS^2_x/\lambda)}{\delta})} + \lambda^{1/2}S_{\theta}$; $S_x$ and $S_{\theta}$ are the upper bounds on $||x||$ and $||\hat{\theta}_{t,a}||$ respectively. $\textbf{1}(.)$ is an indicator function evaluating to one if its argument is true and zero otherwise.

We refer to this as \emph{mOFUL}: see Algorithm \ref{mOFUL}. As in \citep{Abbasi11}, we can obtain the regret bound of the mOFUL algorithm.
\begin{theorem}
With probability at least $1 -\delta$, Algorithm \ref{mOFUL} achieves the regret $R_T  \le \mathcal O(\sqrt{KT})$, where $K$ is the number of all actions.
\end{theorem}
\begin{minipage}[t]{6.8cm}
\vspace{2pt}
\begin{algorithm}[H]
\caption{mOFUL algorithm}
\label{mOFUL}
\textbf{Input}: \\
\begin{algorithmic}[1]
\FOR{$t =1$ to $T$}
    \STATE receive context $x_t$
    \STATE $a_t, \tilde{\theta}_{t, a_t} = \text{argmax}_{a \in \mathcal A, \theta_a \in C_{t-1,a}} \langle x_t, \theta_a \rangle$
    \STATE play action $a_t$ and receive $r_t$
    \STATE update $C_{t,a}$  by Eq(\ref{eq_11})
    \STATE
    \STATE
\ENDFOR
\end{algorithmic}
\end{algorithm}
\end{minipage}%
\hfil 
\begin{minipage}[t]{6.8cm}
\vspace{2pt}
\begin{algorithm}[H]
\caption{$\epsilon$-mOFUL algorithm}
\label{eps-mOFUL}
\textbf{Input}: the offline dataset $S$; the set $\mathcal L$ \\
\begin{algorithmic}[1]
\FOR{$t =1$ to $T$}
    \STATE receive context $x_t$
    \STATE $a_t, \tilde{\theta}_{t,a_t} = \text{argmax}_{a \in \mathcal A, \theta_a \in C_{t-1,a}} \langle x_t, \theta_a \rangle$
    \IF{$a_t \in \mathcal A \setminus \mathcal L$ }
        \STATE play action $a_t$ and receive $r_t$
        \STATE update $C_{t,a}$ by Eq(\ref{eq:12})
    \ENDIF
\ENDFOR
\end{algorithmic}
\end{algorithm}
\end{minipage}
\subsection{$\epsilon$-mOFUL Algorithm}
A natural approach to improve Algorithm \ref{mOFUL} is to use the offline data with an expectation to reduce both the number of reward calls (online explorations) and the regret. We base this method on using the learned model parameters of several actions
for whom we have "enough" offline data from the logging policy $\mu$. We assume that there is a subset $\mathcal L$ of actions, such that there are $\epsilon > 0$ and $\delta > 0$ such that $|\hat{\theta}_a - \theta^*_a| \le \epsilon$ holds w.p. at least $1 -\delta$ jointly over all $a \in \mathcal L$, where $\hat{\theta}_a$ denotes the learned model parameter of action $a$. We call this set of actions $\epsilon$-\emph{supported actions}.

Unlike the fully online mOFUL algorithm, we here uses only online explorations to learn reward parameters of the remaining actions, i.e., $\mathcal A \setminus \mathcal L$. The update of confidence interval $C_{t,a}$ now is as follows:
\begin{equation} \label{eq:12}
  C_{t,a}=\begin{cases}
    \{\hat{\theta}_a\}, & \text{if $a \in \mathcal L$},\\
    \{\theta \in \mathbb{R}^d: ||\hat{\theta}_{t,a} - \theta||_{\overline{V}_{t,a}} \le \beta_t(\delta) \} , & \text{if otherwise},
  \end{cases}
\end{equation}
where $\hat{\theta}_a$ is the learned reward for $a \in \mathcal L$ as defined above, $\hat{\theta}_{t,a}$ is defined as in Eq(\ref{eq_10}); $\overline{V}_{t,a} = \lambda I + \sum_{i=1}^{t} x_ix_i^T \mathbf{1}(a_i = a)$, $\sqrt{\beta_t(\delta)} = \sigma \sqrt{d\text{log}(\frac{(K-L)(1 + tS^2_x/\lambda)}{\delta})} + \lambda^{1/2}S_{\theta}$; $S_x, S_{\theta}$ are upper bounds of $||x||$ and  $||\theta^*_a||$. We can see that for action $a \in \mathcal L$, the confidence interval $C_{t,a}$ contains a unique value which is the learned reward from offline data, and does not change over time $t$. This allows to reduce the computation in optimization step at line 4 in each iteration.

We refer to this as \emph{$\epsilon$-mOFUL} and describe in Algorithm \ref{eps-mOFUL}. We next provide a cumulative regret bound for Algorithm \ref{eps-mOFUL}.
\begin{theorem}
Assume that $||x|| \le S_x$. With probability at least $1 -\delta$, the $\epsilon$-mOFUL algorithm achieves a regret $$R_T  \le \mathcal O(\epsilon S_x(T- T') + \sqrt{(K-L)T'}),$$
where $T' = |\{1 \le t \le T: a_t \in \mathcal A \setminus \mathcal L\}|$ and $L = |\mathcal  L|$.
\label{therem2}
\end{theorem}
\paragraph{Remark 1.} In Theorem \ref{therem2}, $T'$ denotes the number of reward calls up to $T$ iterations of the Algorithm \ref{eps-mOFUL}. $T-T'$ is therefore the number of rewards calls that were saved due to the offline dataset. For each reward call saving, the first term in the regret bound still incurs a regret that is linear to $\epsilon$. Therefore, the smaller the $\epsilon$, the smaller is the first term. The second term in the regret bound is the regret due to online reward calls and grows only sublinearly with $T' \le T$. The larger the set of $\epsilon$-supported actions, the smaller is this regret term.
\paragraph{Remark 2.} Some offline model-based learning methods e.g., tensor decomposition \citep{anandkumar14b,Hong20} allow for $\epsilon, \delta$ to be arbitrarily small as the size of offline dataset increases. More simply, under linear reward models, we also estimate directly the reward parameters of actions in $\mathcal L$ as $\hat{\theta}_a = (\frac{1}{N_a}\sum_{i=1}^{N_a}x^{\mu}_i(x^{\mu}_i)^T)^{-1}(\frac{1}{N_a}\sum_{i=1}^{N_a}x^{\mu}_ir^{\mu}_i)$, where $N_a = \sum_{i=1}^{n} \mathbf{1}(a^{\mu}_i = a)$ is the number of data supporting action $a$ from dataset $S$. Using the assumption that contexts $\{x^{\mu}_i\}_{i=1}^n$ of $S$ are i.i.d samples from a distribution $P_x$, we can obtain the following proposition:
\begin{proposition}
For an action  $a \in \mathcal A$. The following result holds almost surely,
$$\text{lim}_{N_a \rightarrow \infty} (\frac{1}{N_a}\sum_{i=1}^{N_a}x^{\mu}_i(x^{\mu}_i)^T)^{-1}(\frac{1}{N_a}\sum_{i=1}^{N_a}x^{\mu}_ir^{\mu}_i) = \theta^*_a.$$
\label{prop1}
\end{proposition}
If we choose $\epsilon = \text{max}_{a \in \mathcal L}|\hat{\theta_a} - \theta^*_a|$ the Proposition \ref{prop1} implies that as $N_a$ (the size of $\mathcal L$) increases, $\epsilon$ can be arbitrarily small.

If $T-T' \le \mathcal O(1/(S_x^2 \epsilon^2))$, the first term of the upper bound is dominated by the second term. Thus, $R_T = \mathcal O(\sqrt{(K-L)T'}) \le \mathcal O(\sqrt{KT})$. In this case, Algorithm \ref{eps-mOFUL} improves Algorithm \ref{mOFUL} in the three aspects: online explorations, the regret, and the computation at optimization steps. In extreme cases, when $L = 0$, algorithm \ref{eps-mOFUL} degenerates to algorithm \ref{mOFUL} and when $L = K$ which implies that the condition at line 5 is unsatisfied for every $t$ and thus no reward call is needed. In this case, if $T -T' \le \mathcal O(1/(S_x^2 \epsilon^2))$, the regret of algorithm \ref{eps-mOFUL} is $R_T = \mathcal O(\sqrt{T}/ S_x)$ which is sublinear in $T$.
\paragraph{Remark 3.} In practice, the set $\mathcal L$ may be built by for example choosing $L$ actions with the maximum values of $N_a = \sum_{i=1}^{n} \mathbf{1}(a^{\mu}_i = a)$ which is the number of data supporting action $a$ from dataset $S$. Then $\epsilon$ is determined as $\epsilon = \text{max}_{a \in \mathcal L} |\hat{\theta}_a - \theta^*_a|$. By this way, we can guarantee that all actions in $\mathcal L$ are $\epsilon$-supported actions. We note that $\epsilon$ is not required to be known for Algorithm \ref{eps-mOFUL}, but only participates in the upper bound of $R_T$ of Algorithm \ref{eps-mOFUL}. The value of $L$ can be selected based on the histogram of actions from $S$ showing $N_a$ for each action $a$. If plotted in descending order, $L$ is the value at which $N_a$ starts to diminish fast.
\subsection{$\epsilon$-mOFUL-IPS Algorithm}
Algorithm \ref{eps-mOFUL} makes use of availability of the $\epsilon$-supported actions to reduce the online explorations. In this section, we propose another new policy learning algorithm which can obtain further reduce online explorations while keeping a sublinear regret even without the existence of such $\epsilon$-supported actions. This algorithm exploits "good" context-action pairs from the offline dataset S.

We consider event $(x^{\mu}_i, a^{\mu}_i, r^{\mu}_i)$ in the offline dataset $S$. For a context $x^{\mu}_i$, a good learning algorithm should suggest some action $a$ such that the returned reward is at least $r^{\mu}_i$, else this  learnt policy  is worse than the logging policy. This observation suggests that if the estimated rewards for all actions by an algorithm is less than $r^{\mu}_i$, then the algorithm should use the action $a^{\mu}_i$ and $r^{\mu}_i$ of the logging policy $\mu$, instead of wasting a reward call.

Further, instead of comparing the estimated rewards for all actions against $r^{\mu}_i$ of $\mu$, we can compare them with a value higher than $r^{\mu}_i$ to further reduce the number of unnecessary reward calls.  This value is the estimated reward of a policy $\mu^+$, which is computed as $r^{\mu}_i\text{min}\{\frac{\pi^+(a^{\mu}_i|x^{\mu}_i)}{\mu(a^{\mu}_i|x^{\mu}_i)}, M\}$, and the policy $\mu^+$ is taken as an optimal policy in the restricted space $\Pi^+ \subseteq \Pi$ containing a set of policies $ \pi : \mathcal X \rightarrow \mathcal A$ where
\begin{equation*}
  \pi(a|x) =\begin{cases}
    0, & \text{if $a \in \mathcal U(x, \mu)$},\\
    > 0, & \text{if otherwise},
  \end{cases}
\end{equation*}
and $M$ is clipping constant of the off policy clipped IPS estimator \citep{Strehl10,swaminathan15a} which is defined as:
$$\hat{R}^M(\pi) = \frac{1}{|S|} \sum_{i=1}^{|S|}  r^{\mu}_i \text{min} \{\frac{\pi(a^{\mu}_i|x^{\mu}_i)}{\mu(a^{\mu}_i|x^{\mu}_i)}, M\}.$$
The policy space $\Pi^+$ contains a class of policies close to logging policy $\mu$ in the sense that they share the same supported actions (also same unsupported actions) as $\mu$. Thus, policies in $\Pi^+$ have full support from $\mu$ and it implies that we can find an optimal policy in $\Pi^+$ by using the clipped IPS estimator without online explorations: $\pi^+ = \text{argmax}_{\pi \in \Pi^+}\hat{R}^M(\pi)$. We refer to this algorithm as \emph{$\epsilon$-mOFUL-IPS} and describe it in Algorithm \ref{eps-mOFUL-IPS}. The comparison of estimated rewards by the algorithm and the estimated reward of the policy $\pi^+$ is performed at line 6. If contexts $x \not \in S$, the algorithm is similar to Algorithm \ref{eps-mOFUL} with confidence intervals updated by Eq(\ref{eq:12}). Following the above observations, to make use of contexts in $S$. If $|S| \ge T$, we use any $T$ contexts from $S$ as contexts of \ref{eps-mOFUL-IPS}. Otherwise, we use all $|S|$ contexts in $S$ and $T -|S|$ contexts are sampled randomly from $P_x$.
\begin{algorithm}[tb]
\caption{$\epsilon$-mOFUL-IPS algorithm}
\label{eps-mOFUL-IPS}
\textbf{Input}: offline dataset S, the set $\mathcal L$ of $\epsilon$-supported actions \\
\begin{algorithmic}[1]
\FOR{$t =1$ to $T$}
    \STATE receive context $x_t$
    \STATE $a_t, \tilde{\theta}_{t, a_t} = \text{argmax}_{a \in \mathcal A, \theta_a \in C_{t-1,a}} \langle x_t, \theta_a \rangle$
    \IF {$x_t \in S$ and assume $x_t$ corresponding to some $(x^{\mu}_i, a^{\mu}_i, r^{\mu}_i)$ of $S$}
        \IF{$\langle x_t, \tilde{\theta}_{t, a_t} \rangle > r^{\mu}_i\text{min}\{\frac{\pi^+(a^{\mu}_i|x^{\mu}_i)}{\mu(a^{\mu}_i|x^{\mu}_i)}, M\}$ and $a_t \in \mathcal A \setminus \mathcal L$}
             \STATE play action $a_t$ and receive $r_t$
             \STATE update $C_{t,a}$ by Eq(\ref{eq:12})
        \ENDIF
    \ELSE
        \IF{$a_t \in \mathcal A \setminus \mathcal L$}
             \STATE play action $a_t$ and receive $r_t$
             \STATE update $C_{t,a}$ by Eq(\ref{eq:12})
        \ENDIF
    \ENDIF
\ENDFOR
\end{algorithmic}
\end{algorithm}

We next provide a theoretical analysis for Algorithm \ref{eps-mOFUL-IPS}. Given logged dataset $S$ with size $n$, we define
$$u_{\pi}^i = r^{\mu}_i\text{min} \{ M, \frac{\pi(a^{\mu}_i|x^{\mu}_i)}{\mu(a^{\mu}_i|x^{\mu}_i)}\}, \text{ } \overline{u}_{\pi} = \frac{1}{n}\sum_{i=1}^{n}u^i_{\pi} \text{ , } Var(\pi) = \sum_{i=1}^{n} (u^i_{\pi} - \overline{u}_h)^2/(n-1),$$
and a reward function class $\mathcal F_{\Pi} = \{f_{\pi}: \mathcal X \times \mathcal A \rightarrow [0,1] \}$. The covering number $N(1/n, \mathcal F_{\Pi}, 2n)$ is the size of the smallest cardinality subset $A_0 \in \mathcal F_{\Pi}$ such that $\mathcal F_{\Pi}$ is contained in the union of balls of radius $1/n$. We get the following upper bound for $R_T$.
\begin{theorem}
Assume that $||x|| \le S_x$. For any $\epsilon \ge 0$ and $T > 0$, with probability at least $1- \delta$, the $\epsilon$-mOFUL-IPS algorithm achieves an upper bound on the regret
$$R_T \le \mathcal O(\epsilon S_x(T-T'-T'') + \sqrt{(K-L)T'} + \sqrt{Var(\pi^+)log(N(1/T'', \mathcal F_{\Pi}, 2T'')/\delta)T''}),$$
where $T' = |\{  1 \le t \le T: \langle x_t, \tilde{\theta}_{t, a_t} \rangle > r^{\mu}_i\text{min}\{\frac{\pi^+(a^{\mu}_i|x^{\mu}_i)}{\mu(a^{\mu}_i|x^{\mu}_i)}, M\} \text{ and } a_t \in \mathcal A \setminus \mathcal L \}|$; $T'' = |\{  1 \le t \le T: \langle x_t, \tilde{\theta}_{t, a_t} \rangle \le r^{\mu}_i\text{min}\{\frac{\pi^+(a^{\mu}_i|x^{\mu}_i)}{\mu(a^{\mu}_i|x^{\mu}_i)}, M\} \}|$; $Var(\pi^+)$ and $N(1/T, \mathcal F_{\Pi}, 2T)$ are defined as above;
\label{theorem3}
\end{theorem}
\paragraph{Remark 4.} Following \citep{zhou2018offline}, $\sqrt{Var(\pi^+)log(N(1/T, \mathcal F_{\Pi}, 2T)/\delta)T})$ is sublinear in $T$. If $T-T'-T'' \le \mathcal O(1/(S_X^2 \epsilon^2))$, the first term of Theorem \ref{theorem3} is dominated, hence the regret of the $\epsilon$-mOFUL-IPS algorithm achieves a sublinear rate in $T$. In Theorem \ref{theorem3}, $T'$ denotes the number of reward calls up to $T$ iterations of the Algorithm \ref{eps-mOFUL-IPS}. Compared to the number of reward calls of the Algorithm \ref{eps-mOFUL}, the Algorithm \ref{eps-mOFUL-IPS}'s one is potentially smaller.
\paragraph{Remark 5.} If $L = 0$ i.e., there is no availability of $\epsilon$-supported actions, then
$T = T' + T''$ which implies that the first term is removed. The algorithm still guarantees a sublinear rate.
\vspace*{-4mm}
\section{Experiments}
\begin{figure*}[ht]
\centering
\subfigure[$n^{UA} = 0.2$, $L = 5$]{\includegraphics[scale=1.0,width=.32\textwidth, height= .13\textheight]{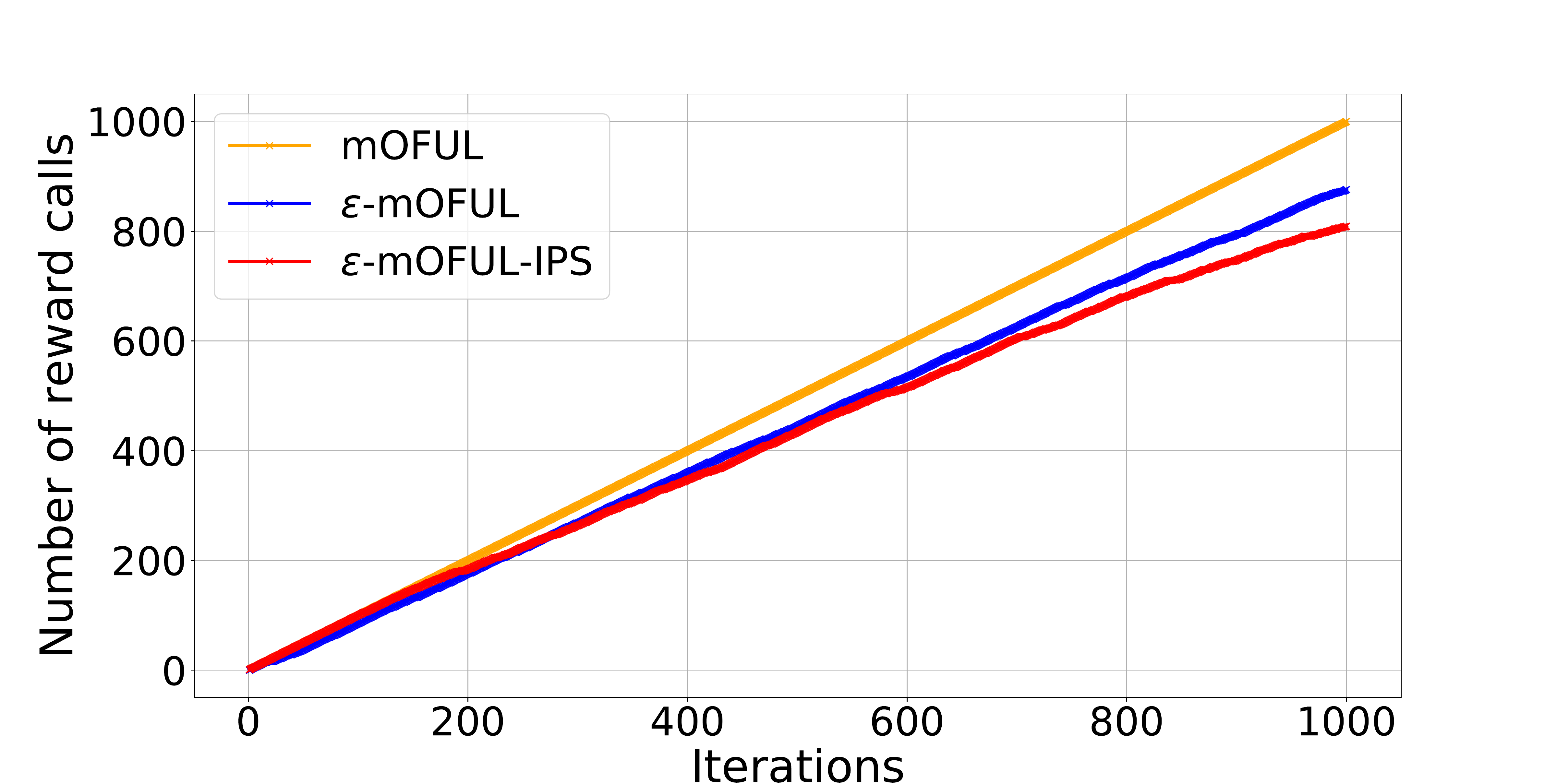}
}\hfill
\subfigure[$n^{UA} = 0.2$, $L = 10$]{\includegraphics[scale=1.0,width=.32\textwidth, height= .13\textheight]{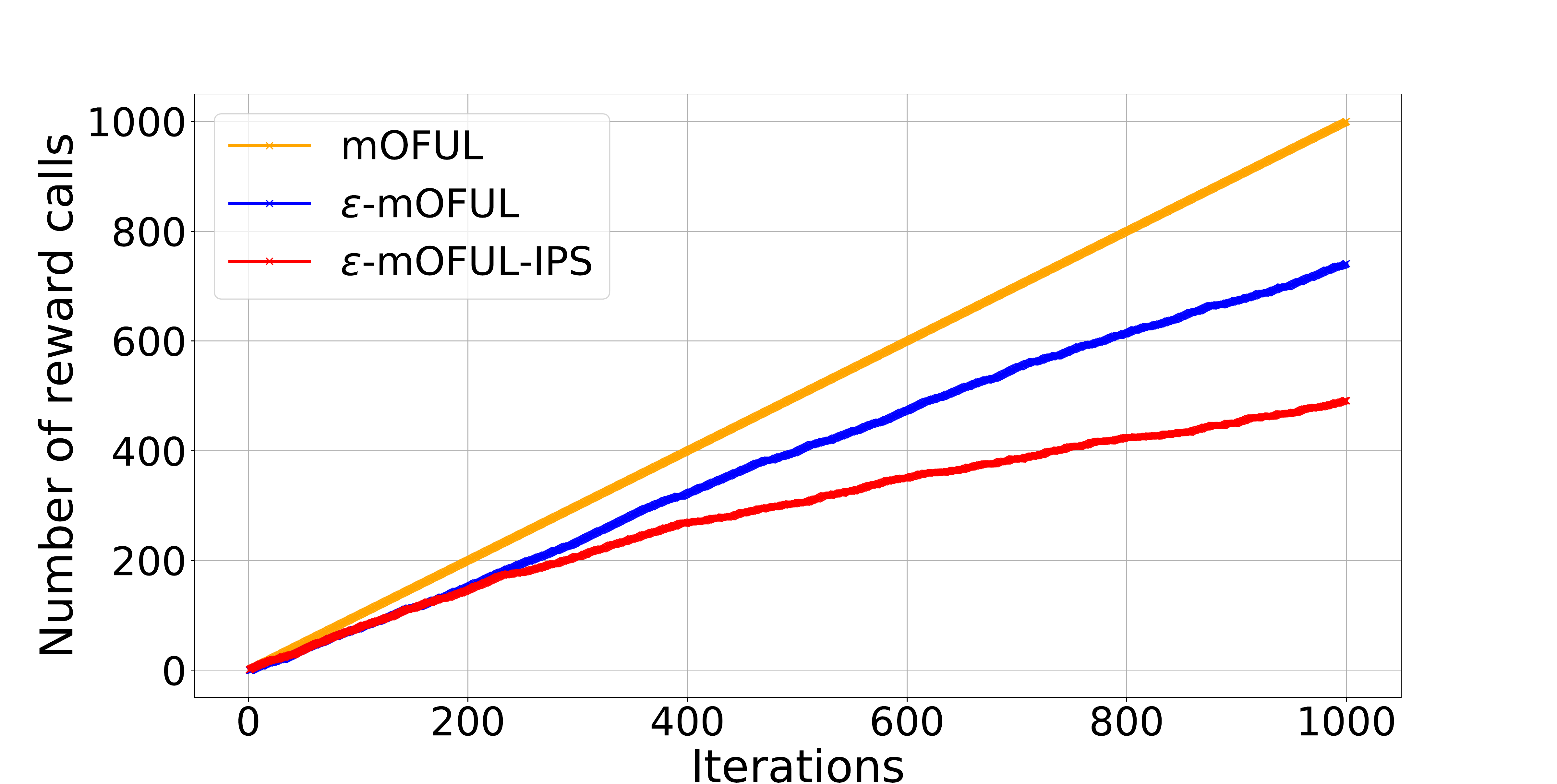}
}\hfill
\subfigure[$n^{UA} = 0.2$, $L = 12$]{\includegraphics[scale=1.0,width=.32\textwidth, height= .13\textheight]{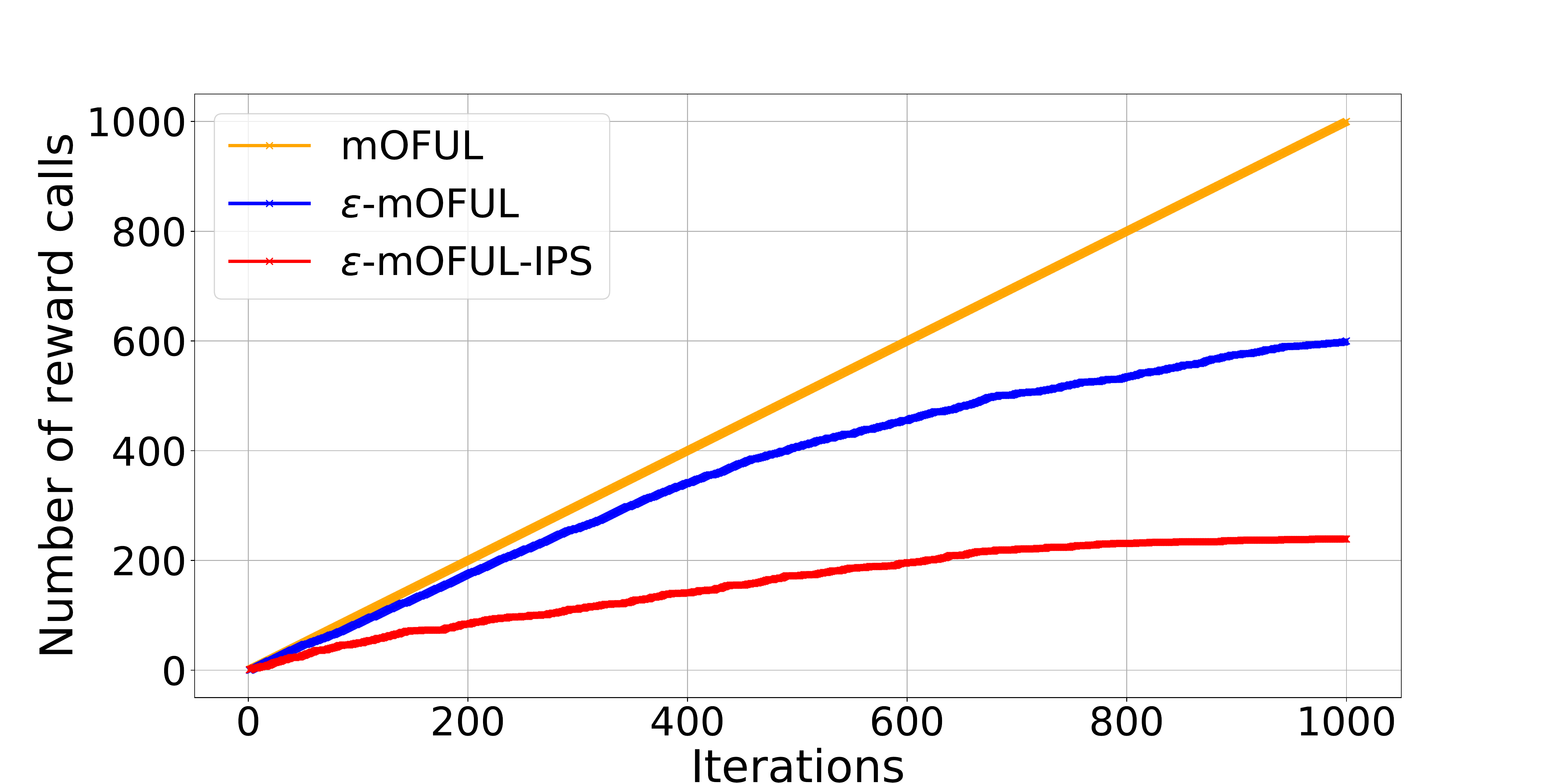}
}
\vspace*{-3mm}
\caption{The results for a fixed number of average unsupported action $n^{UA}$ and varying number of supported actions L. Our proposed offline-online methods are better than fully online method and the reward call saving increases with increasing $L$.}
\label{fig1}
\vspace*{-3mm}
\end{figure*}
We empirically compare the performance of our two algorithms $\epsilon$-mOFUL and  $\epsilon$-mOFUL-IPS to two baselines: (1) the fully online learning algorithm, mOFUL, as described in Algorithm 1; (2) the fully offline learning algorithm based on policy restriction in \citep{Sachdeva2020}, denoted by OPR.

\paragraph{Support Deficiency Setting}
Let $n^{UA} = \mathbb{E}_{x \in \mathcal X}[\frac{|\mathcal U(x, \mu)|}{K}]$ 
be the average number of unsupported actions of logging policy $\mu$ given a set of context $\mathcal X$ and a set of actions $\mathcal A = \{1, 2,.., K\}$. To create a logging policy $\mu$ at some level, for example $ n^{UA} = 0.5$, we do the following: for each $x$, we randomly select a set of  $50\%$ of actions in $\mathcal A$. We consider this  set as $\mathcal U(x, \mu)$, and create the logged data by considering the actions in the complement of this set. To build a data set from the logging policy $\mu$, for each $x$, we select uniformly at random an action $a \in \mathcal U^c(x, \mu)$ which is in the set of supported actions for $x$. In this manner, for each $a \in \mathcal U(x, \mu)$: $\mu(a |x) = 0$ and for each $a \in \mathcal U^c(x, \mu)$: $\mu(a |x) = \frac{2}{K}$.
\vspace*{-3mm}
\subsection{Synthetic Experiments}
\begin{wrapfigure}{r}{0.4\textwidth}
\begin{center}
\vspace*{-10mm}
\centerline{\includegraphics[scale=1.0,width=.4\textwidth]{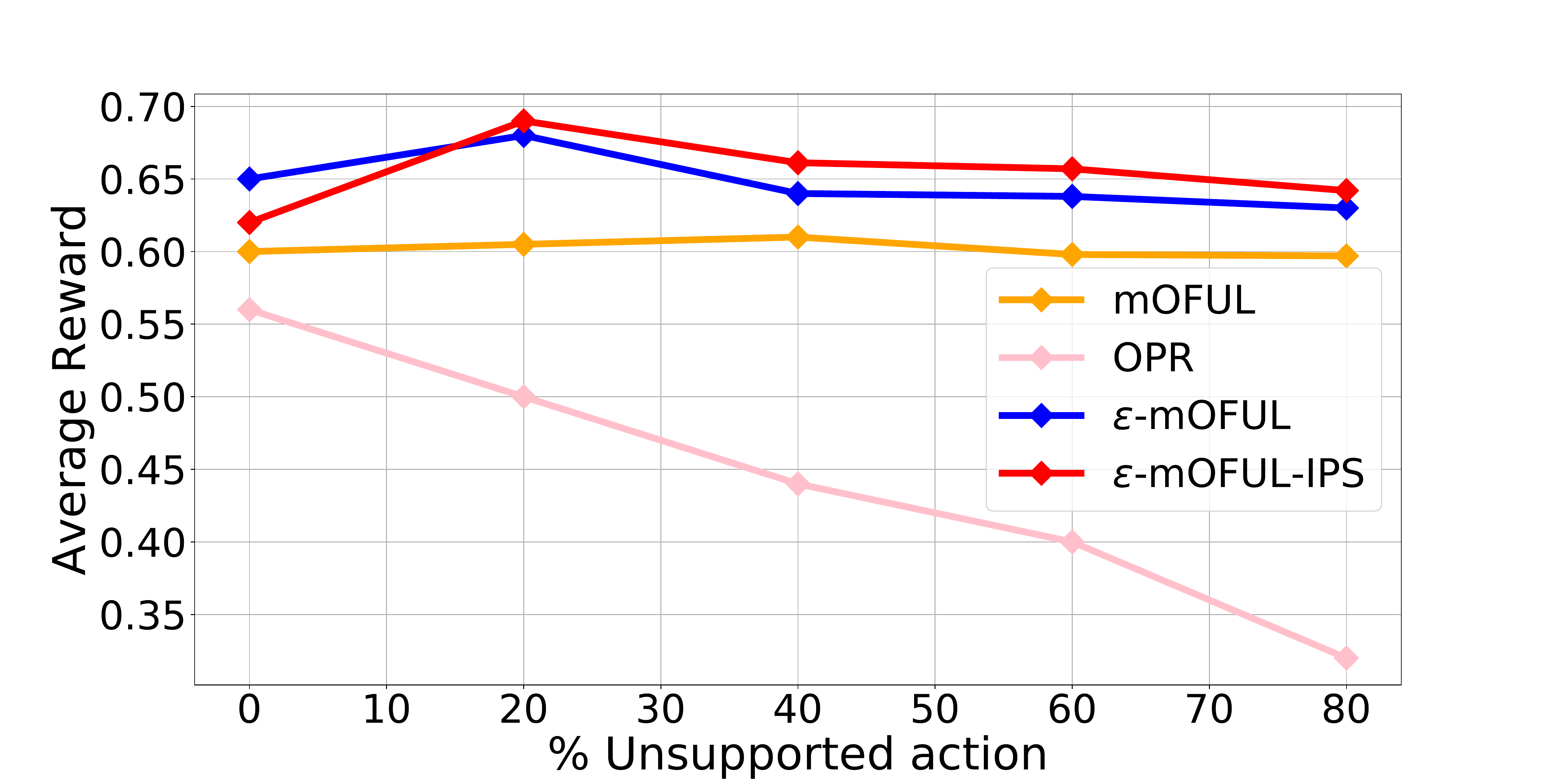}}
\caption{The average reward of different learning algorithms vs. $\%$ of unsupported actions, $n^{UA}$ for logging policy $\mu$. We performed experiments for $n^{UA} = [0, 0.2, 0.4, 0.6, 0.8]$.}
\label{fig3}
\end{center}
\vspace*{-8mm}
\end{wrapfigure}
We first experiment with synthetic contextual bandits with $|\mathcal A| = 20$. We use a linear reward model: $r(x,a) = \langle x, \theta^{*}_a \rangle + \eta$, where $\theta^{*}_a$ is the parameter for action $a$, drawn i.i.d from $\mathcal N(0, 1)$, and $\eta$ is the noise, drawn i.i.d from $\mathcal N(0, \sigma^2)$, where $\sigma = 2$. 
The contexts $x$ are sampled uniformly at random in $[0, 1]$. For $\epsilon$-mOFUL and $\epsilon$-mOFUL-IPS, we let $\epsilon = 0.05$. To build the set $\mathcal L$ of $\epsilon$-supported actions, we choose a set of any $L$ actions and sample $\hat{\theta}_a$ from $\mathcal N(\theta^*_a, \epsilon^2)$ for each action in $\mathcal L$.

We create a dataset for logging policy $\mu$ which is large enough to support each action in $\mathcal L$. OPR is learnt using a direct loss minimization (DLM) algorithm of \citet{Hazan10} as in \citep{Dud11}. For our proposed $\epsilon$-mOFUL-IPS, we learn $\pi^+$  via the linear programming (LP):
$\text{max}_{\pi \in \Pi} \sum_{(x, a, r) \in S} r_i \frac{\pi(a|x)}{\mu(a|x)} \text{ subject to } 0 \le \frac{\pi(a|x)}{\mu(a|x)} \le M,$
where the clipping constant $M$ is set to the ratio of the $90$-th percentile to the $10$-th percentile propensity score observed in the training set.

The efficiency of our proposed algorithms is measured via (1) the number of reward function calls, and (2) the estimated average reward. Lower the number of reward function calls, and higher the value of estimated average reward, the greater the efficiency of the algorithm is.
\subsubsection{On the number of online reward function calls}
We test the influence of $L$ and $n^{UA}$ on the number of reward function calls.
\paragraph{On the influence of $L$ parameter:} We fix $n^{UA} = 0.2$ and consider three different values of $L$: $L = 5, L = 10$ and $L = 12$. As shown in Section 5, all these algorithms converge sub-linearly. However, while mOFUL needs to use one reward call at each iteration, our algorithms uses fewer reward calls as $L$ increases. This is presented in Figure \ref{fig1}. Proposed $\epsilon$-mOFUL-IPS algorithm uses the least number of reward calls because it only makes a reward call if the chosen action does not belong to $\mathcal L$ and satisfies the condition at Line 6 in Algorithm 3. In particular, when $L = 12$, the number of reward calls is only about $20\%$ of mOFUL's calls.
\paragraph{On the influence of $n^{UA}$ parameter:} We provide results in Supplementary Material.
\subsubsection{On the efficiency of the algorithms}
We now fix the number of reward calls and compare the algorithms using the estimated average reward. Each algorithm provides a learnt policy. Given contexts $\{x_1, x_2,...,x_T\}$ and a policy $\pi$, for the algorithms using online explorations, we compute the average reward as $\sum_{i=1}^{T}\frac{1}{T} \langle x_i, \theta^{*}_{\pi (x_i)} \rangle$. For OPR, since the policy $\pi$ is stochastic, the average reward is computed as $\sum_{i=1}^{T}\frac{1}{T} \langle x_i, \theta^{*}_{\pi (a|x_i)} \rangle$, where action $a$ is chosen randomly from $\mathcal U(x_i, \pi)$. We compared estimated rewards of the algorithms with different values of $n^{UA} = 0; 0.2; 0.4; 0.6; 0.8$. We allow $T$ increasing but fix $L = 300$ across all values of $n^{UA}$.  The result is shown in Figure \ref{fig3}. The estimated reward of OPR is the worst across all $n^{UA}$. This is because OPR only searches optimal policies in a restricted subspace closed to the logging policy $\mu$, does not use online exploration, and thus fails to find optimal policies. The performance is worse for higher the level of unsupported actions.  In contrast, mOFUL, $\epsilon$-mOFUL and $\epsilon$-mOFUL-IPS  can find optimal policies for all the levels of  unsupported actions. mOFUL does not use offline data, the estimated rewards are nearly the same for all the levels of unsupported actions. $\epsilon$-mOFUL and $\epsilon$-mOFUL-IPS leverage a large amount of offline data to learn reward model and thus, are more efficient than mOFUL. $\epsilon$-mOFUL-IPS leverages both learnt reward model and an optimal policy $\pi^+$ of the subset of policy space generated from the supported actions of $\mu$, and runs well especially for small levels of  unsupported actions. In summary, our proposed $\epsilon$-mOFUL and especially $\epsilon$-mOFUL-IPS achieves better efficiency than mOFUL given the same reward calls.
\subsection{Multi-class Classification with Bandit Feedback}
\begin{figure*}[ht]
\centering
\subfigure[letter dataset]{\includegraphics[scale=1.0,width=.32\textwidth, height= .13\textheight]{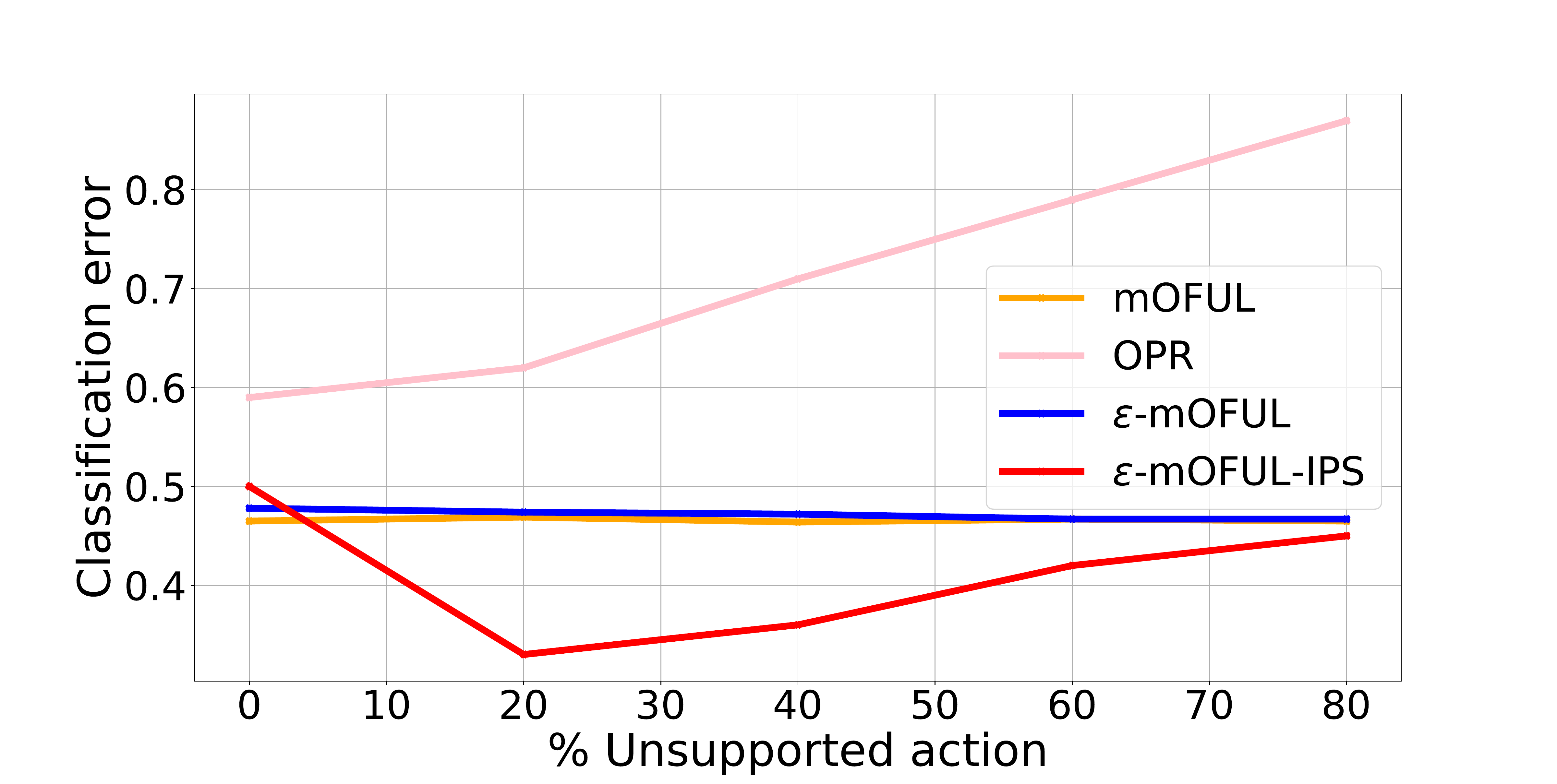}
}\hfill
\subfigure[pendigits dataset]{\includegraphics[scale=1.0,width=.32\textwidth, height= .13\textheight]{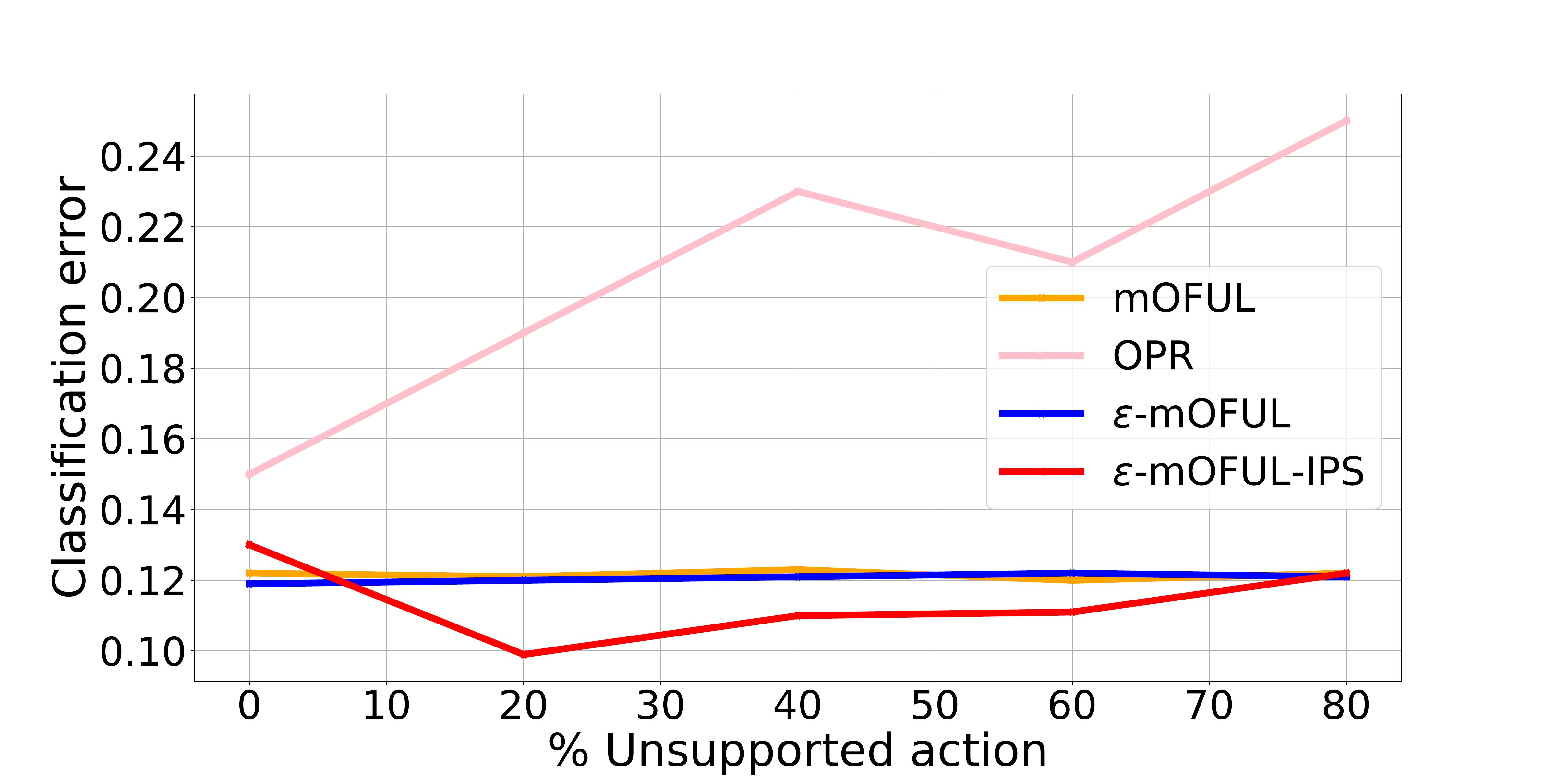}
}\hfill
\subfigure[satimage dataset]{\includegraphics[scale=1.0,width=.32\textwidth, height= .13\textheight]{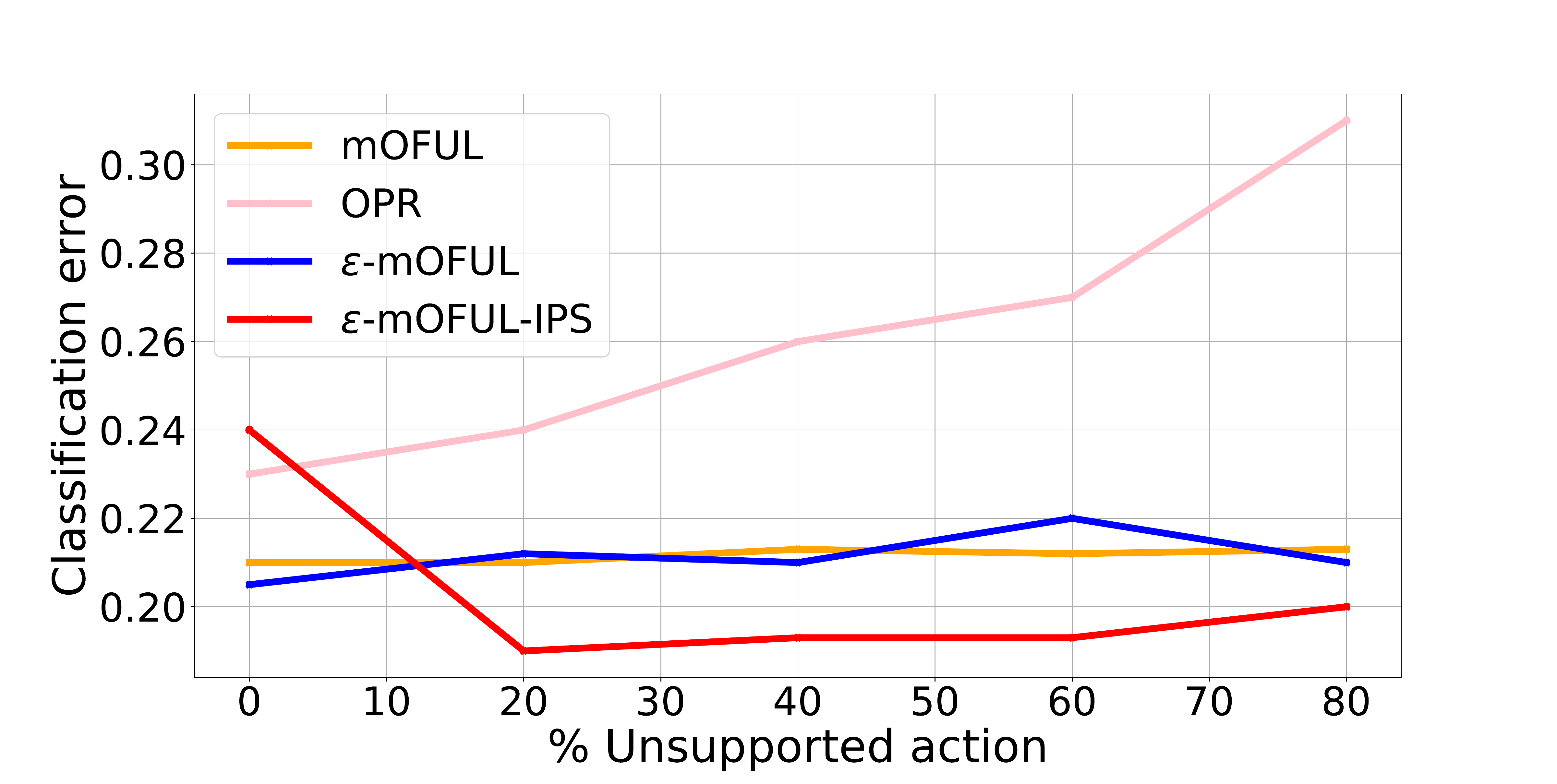}
}
\vspace*{-3mm}
\caption{Classification error for different algorithms as average number of the unsupported actions increases ($n^{UA} = [0, 0.2, 0.4, 0.6, 0.8]$).}
\label{fig4}
\vspace*{-3mm}
\end{figure*}

\begin{table}[ht]
\centering
\caption{Characteristics of benchmark datasets used in Section 5.2.}
\resizebox{8cm}{!}{%
\begin{tabular}[t]{lccc}
\hline Dataset  & letter & pendigits & satimage \\ \hline
Classes ($K$)   &  26    & 10        &   6       \\
Dataset size    & 20000  & 10992     & 6435       \\ \hline
\end{tabular}}
\label{table1}
\vspace*{-4mm}
\end{table}
Next, we learn optimal policies for classification tasks using the three UCI datasets previously considered for offline-policy learning in \citet{Dud11}. The datasets are described in Table \ref{table1}. We assume that data are drawn from a fixed distribution: $(x, c) \sim D$, where $x \in \mathcal X$ is a feature vector and $c \in \{1, 2, .., K\}$ is the class label. A typical goal is to find a classifier $\pi : \mathcal X \rightarrow \{1, 2,.., K\}$ minimizing the classification error: $e(\pi) = \mathbb{E}_{(x,a) \sim D} [\textbf{1}(\pi(x) \not = c)]$.

We convert the multi-class classification problem to contextual bandits by treating the labels as actions for a policy $\mu$, and recording the reward of 1 if the correct label is chosen, and 0 otherwise. We turn a data point $(x, c)$ into a classification example $(x, r_1, r_2,..., r_K)$, where $r_a = \textbf{1}(a \not = c)$ is the loss (reward) for predicting $a$. We note that only $r_a$ is revealed. In addition to this noiseless reward model, we also consider a noisy reward model for each data set, which reveals the correct reward with probability 0.5 and outputs a random coin toss otherwise. We use a linear loss model $r(x,a) = \langle x, \theta^{*}_a \rangle + \eta$, where $\theta^{*}_a$ is the unknown parameter of action $a$, and $\eta$ is the noise defined as above. We randomly split data into two sets, one ($70\%$) set to build the logging policy and ($30\%$) set to build algorithms requiring online explorations. The logging policy with different levels of unsupported actions as well as the algorithms are implemented as for the synthetic function except that contexts $x$ are obtained from datasets and the parameters $\theta_a^*$ of actions are assumed to be unknown. The set of $\epsilon$-supported actions is built as Remark 3 mentioned. We show this in Supplementary Material.

We test the performance of the algorithms in terms of the classification error for all three datasets. The results are shown in Figure \ref{fig4}. Similar to the synthetic experiments, our $\epsilon$-mOFUL-IPS algorithm achieves the best performance as its classification error is the lowest among all algorithms consistently on all datasets. OPR is the worst among all algorithms on all datasets. It is because OPR only searches optimal policies in a restricted subspace close to the logging policy $\mu$. This algorithm does not use online explorations and thus fails to find policies that are optimal in the whole policy space. The performance of $\epsilon$-mOFUL and mOFUL are nearly similar because the amount of data is not enough to learn the $\epsilon$-supported actions as well as in the synthetic experiments.
\section{Conclusion}
We study the problem of policy learning with logged data in contextual bandits with deficient support. Due to the deficient support, the policies learnt in the offline setting are biased.  To solve this problem, we combine the offline-policy learning with online explorations. We propose two algorithms. The first algorithm leverages reward models learned from offline data to reduce the number of online explorations. The second algorithm further improves the efficiency of the first algorithm to reduce the number of online interactions by exploiting good context-action pairs in offline data. We perform experiments with both synthetic and real datasets and show the efficiency of our algorithms.
\nocite{langley00}
\bibliographystyle{plainnat}
\bibliography{Hung-research}

\newpage
\appendix

\centerline{ \textbf{\huge Supplementary Material}}

\section{Additional Experiments}
\paragraph{On the influence of $n^{UA}$ parameter} We fix $L = 4$. We vary the average number of unsupported actions, $n^{UA}$ as $n^{UA} = 0.2, n^{UA} = 0.4$ and $n^{UA} = 0.6$. As shown in Figure \ref{fig2}, mOFUL always requires one reward call at each iteration. The efficiency of $\epsilon$-mOFUL does not change across the different values of $n^{UA}$ as the set of $\epsilon$- supported actions and the learnt parameter $\theta^+_a$ are fixed. As expected, when $n^{UA}$ decreases, the number of supported actions increases, therefore $\epsilon$-mOFUL-IPS requires fewer reward calls due to learnt policy $\pi^+$.

\begin{figure*}[ht]
\centering
\subfigure[$n^{UA} = 0.6$, $L = 4$]{\includegraphics[scale=1.0,width=.32\textwidth, height= .13\textheight]{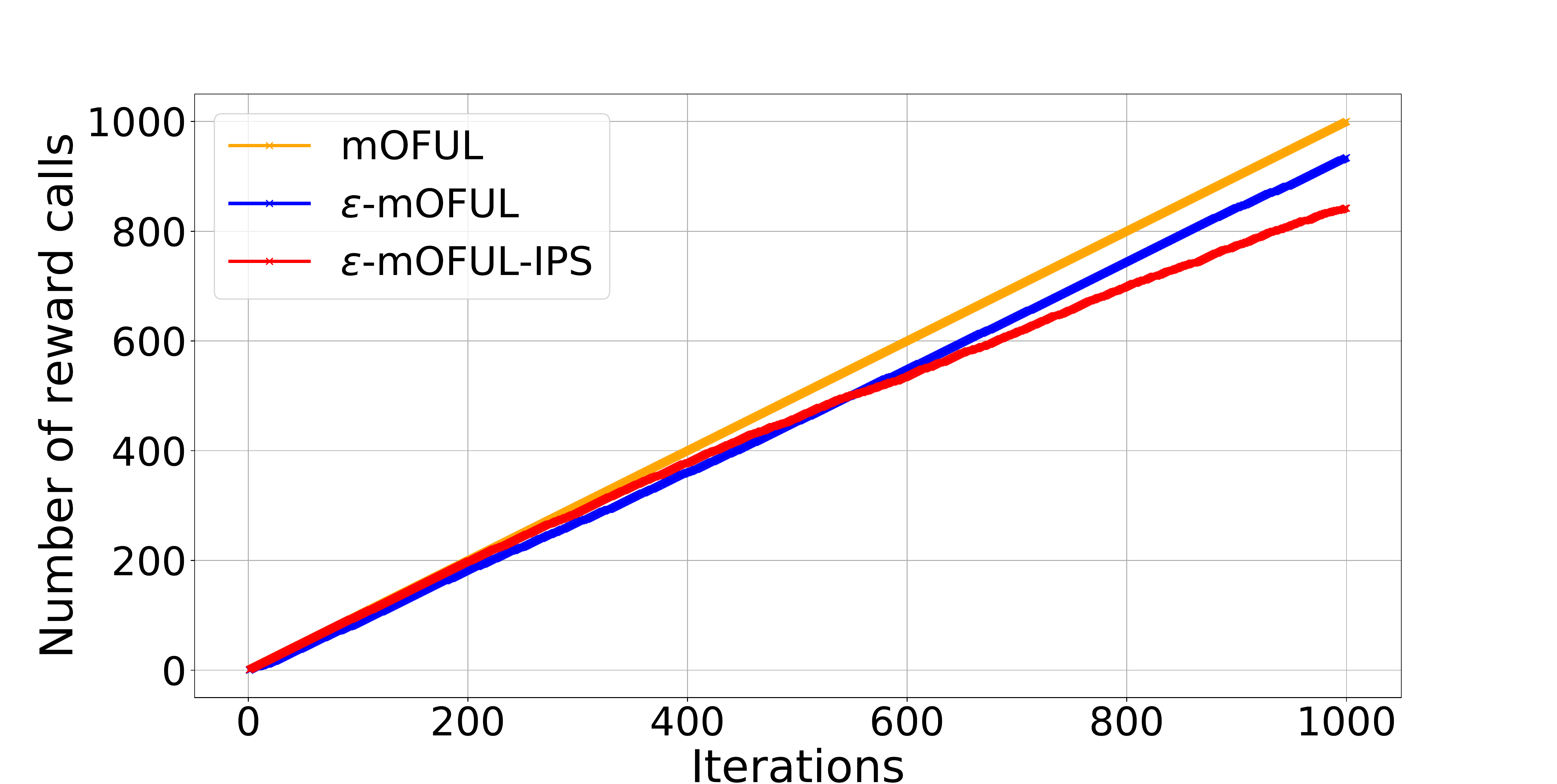}
}\hfill
\subfigure[$n^{UA} = 0.4$, $L = 4$]{\includegraphics[scale=1.0,width=.32\textwidth, height= .13\textheight]{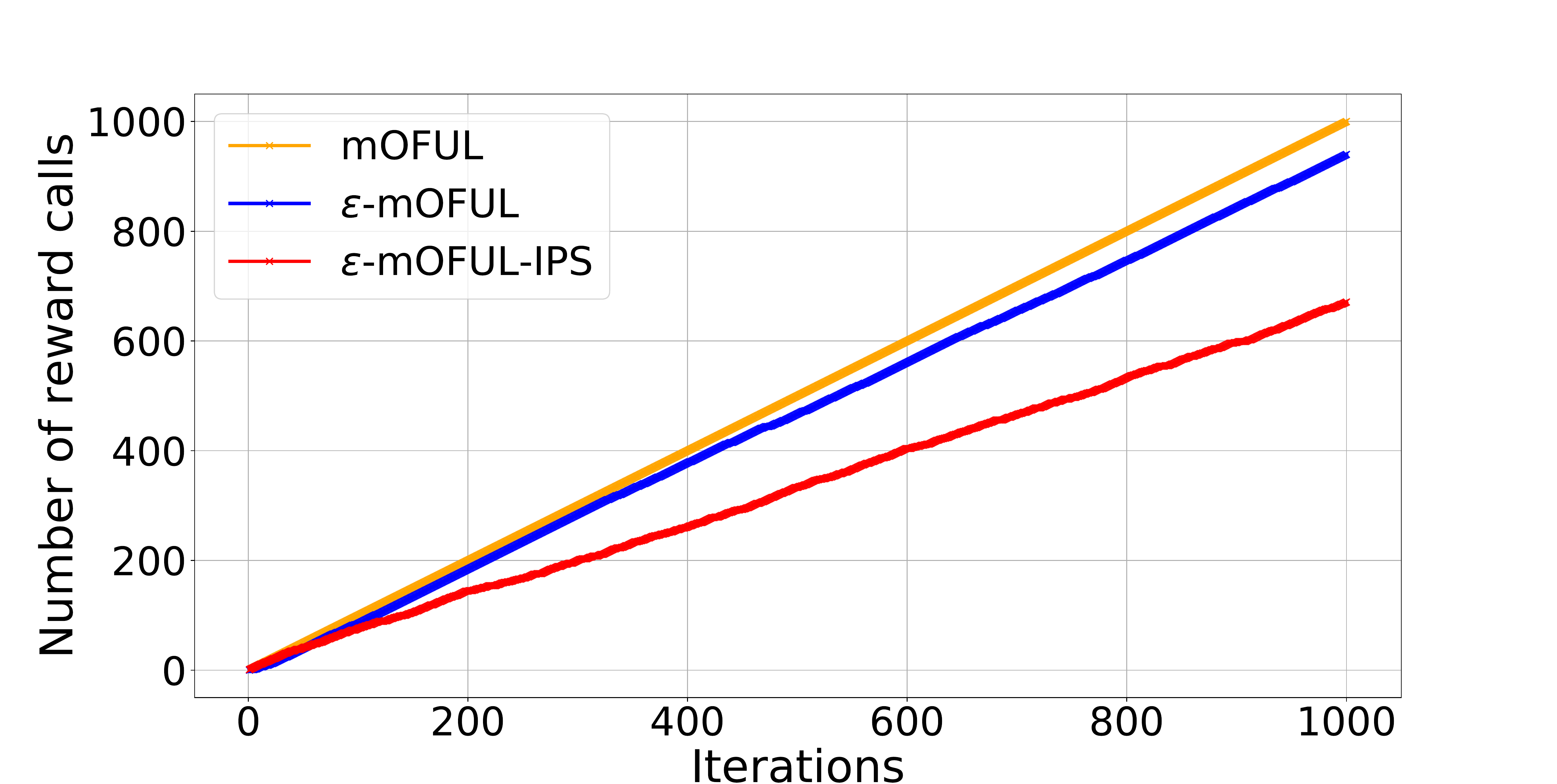}
}\hfill
\subfigure[$n^{UA} = 0.2$, $L = 4$]{\includegraphics[scale=1.0,width=.32\textwidth, height= .13\textheight]{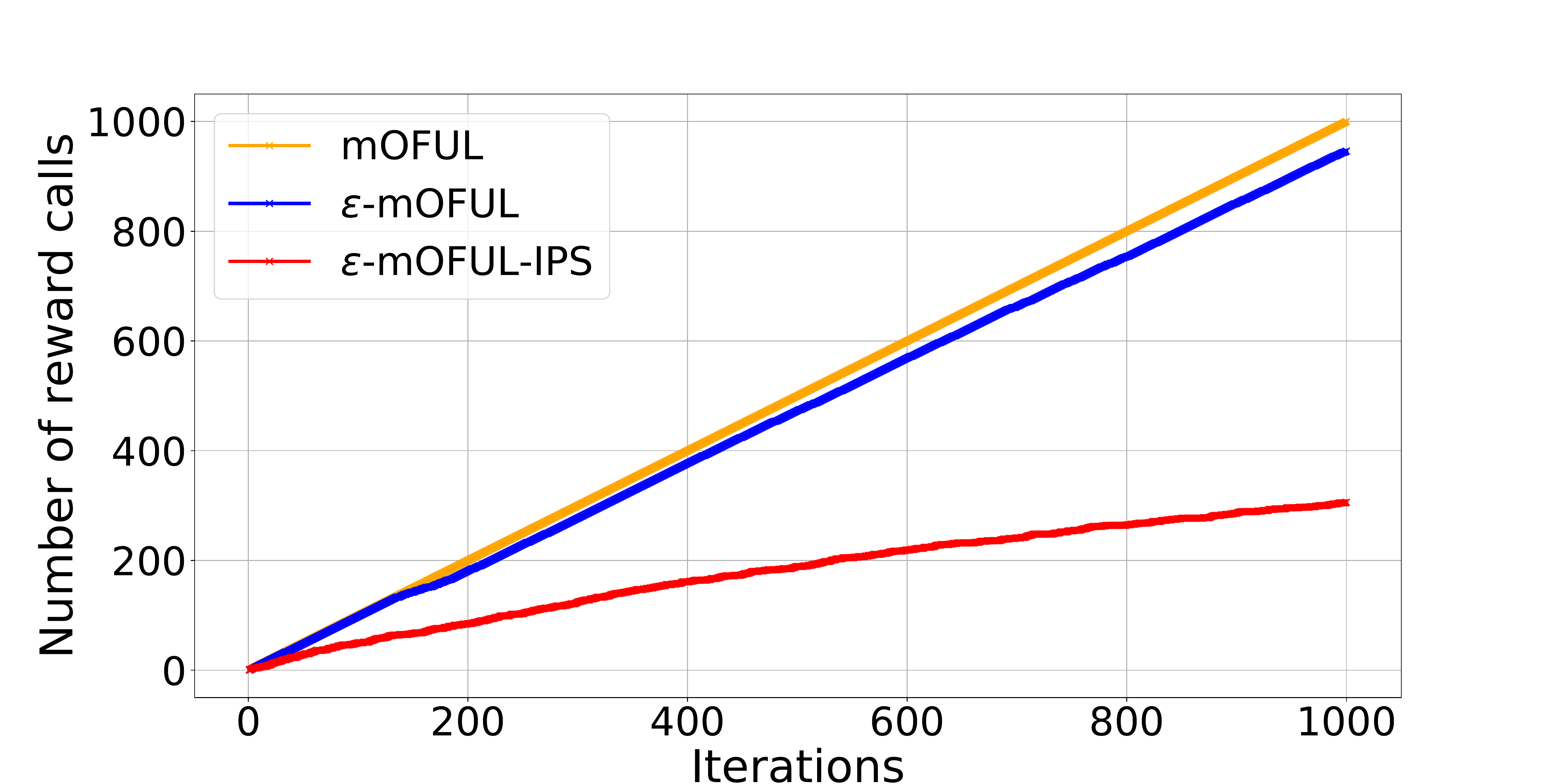}
}
\vspace*{-3mm}
\caption{A fixed $L$ and varying number of average unsupported actions $n^{UA}$. As $n^{UA}$ decreases, our proposed offline-online methods increasingly save the number of reward calls.}
\label{fig2}
\vspace*{-3mm}
\end{figure*}

\section{Proof of Theorem 1}
In this section, we provide a proof for Theorem 1. The proof technique is standard and similar to the one of \citep{Abbasi11}.

By adapting the proofs of \citep{Abbasi11}, we achieve the following results under our model.
\begin{lemma}[Based on Theorem 2 of \citep{Abbasi11}]
Let $V = I\lambda$, $\lambda > 0$ and assume that $||\theta^*_{a}|| \le S_{\theta}$ and $||x|| \le S_x$ for every $a \in \mathcal A$ and every $x \in mathcal X$. Then for any $\delta \in (0,1)$, with probability at least $1 -\delta$, for all $a$ and $t$, $\theta^*_{a}$ lies in the set
$$C_{t,a}= \{\theta \in \mathbb{R}: ||\hat{\theta}_{t,a} - \theta||_{\overline{V}_{t,a}} \le \beta_t \},$$
where $\overline{V}_{t,a} = \lambda I + \sum_{i=1}^{t} x_ix_i^T \mathbf{1}(a_i = a)$ and
$\sqrt{\beta_t(\delta)} = \sigma \sqrt{dlog(\frac{K-(1 + tS^2_x/\lambda)}{\delta})} + \lambda^{1/2}S_{\theta}$.
\end{lemma}
\begin{lemma}[Based on Lemma 11 of  \citep{Abbasi11}]
Let $\{x_t\}_{t =1}^{\infty}$ be a sequence in $\mathbb{R}^d$ and $V_t = V + \sum_{i=1}^{t}x_ix_i^T$. Assume that $||x_t|| \le S_x$ for all $t$. Then
$$\sum_{i=1}^{t}\text{min}(||x_i||^2_{V^{-1}_{i-1}}, 1) \le 2log(\frac{det(V_t)}{det(V)}) \le 2dlog(\frac{trace(V) + tS_x^2}{d}) - 2log(det(V)).$$
\end{lemma}

Following the technique of \citep{Abbasi11}, we first analyze the regret $r_t = \langle x_t, \theta^*_{\pi^*(x_t)} \rangle -\langle x_t, \theta^*_{a_t} \rangle$ at each iteration $t$ as follows.

\begin{eqnarray*}
r_t  & = & \langle x_t, \theta^*_{\pi^*(x_t)} \rangle -\langle x_t, \theta^*_{a_t} \rangle \\
& \le & \langle x_t, \tilde{\theta}_{t, a_t} \rangle -\langle x_t, \theta^*_{a_t} \rangle \\
& \le & \langle x_t, \tilde{\theta}_{t, a_t} -  \theta^*_{a_t} \rangle \\
& \le & \langle x_t, \tilde{\theta}_{t, a_t} -  \hat{\theta}_{t-1, a_t} \rangle + \langle x_t, \hat{\theta}_{t-1,a_t} -  \theta^*_{a_t} \rangle\\
& \le & ||x_t||_{\overline{V}^{-1}_{t-1, a_t}} ||\tilde{\theta}_{t, a_t} -  \hat{\theta}_{t-1, a_t}||_{\overline{V}^{-1}_{t-1, a_t}} + ||x_t||_{\overline{V}^{-1}_{t-1, a_t}} ||\hat{\theta}_{t-1,a_t} -  \theta^*_{a_t}||_{\overline{V}^{-1}_{t-1, a_t}} \\
& \le &  2\sqrt{\beta_{t-1}(\delta)} ||x_t||_{\overline{V}^{-1}_{t-1, a_t}}.
\end{eqnarray*}
Next, notice that $r_t \le 2$. Given a $a \in \mathcal A$, we denote $N_t(a) = |\sum_{i=1}^{t} \textbf{1}\{a_i = a\}|$. We can upper bound the cumulative regret $R_T = \sum_{t=1}^{T}\langle x_t, \theta^*_{\pi^*(x_t)} \rangle -\langle x_t, \theta^*_{a_t} \rangle$ after $T$ iterations as follows:
\begin{eqnarray*}
R_T  & \le &  \sqrt{T\sum_{t=1}^{T} r^2_t} \\
& = &  \sum_{t=1}^{T} r_t \\
& \le &  2\sqrt{T\sum_{t=1}^{T} \beta_{t-1}(\delta)\text{min}\{||x_t||_{\overline{V}^{-1}_{t-1, a_t}}, 1\}} \\
& \le &  2\sqrt{T\beta_{T}(\delta)\sum_{a \in \mathcal A} \sum_{i=1}^{N_T(a)} \text{min}\{||x_{I_i(a)}||_{\overline{V}^{-1}_{I_i(a), a}}, 1\}} \\
& \le &  2\sqrt{T\beta_{T}(\delta)} \sqrt{d\sum_{a \in \mathcal A} log(\lambda + \frac{N_T(a)S^2_x}{d})} \\
& \le &  2\sqrt{T\beta_{T}(\delta)} \sqrt{dKlog(\lambda + \frac{TS^2_x}{dK})} \\
& =  & \mathcal O(\sqrt{KT}),
\end{eqnarray*}
where in the forth inequality, we use Lemma 3. In the last inequality, we use Jensen's inequality and the fact that $\sum_{a \in \mathcal A}N_T(a) = T$. We here hide the influence of the number of dimensions $d$. Thus, the theorem is proven.

\section{Proof of Proposition 2}
\begin{proposition}
The following result holds almost surely for all $a \in \mathcal A$.
$$\text{lim}_{N_a \rightarrow \infty} (\frac{1}{N_a}\sum_{i=1}^{N_a}x_ix_i^T)^{-1}(\frac{1}{N_a}\sum_{i=1}^{N_a}x_ir_i) = \theta^*_a,$$
where $N_a = \sum_{i=1}^{N} \mathbf{1}(a_i = a)$, and the $x_i, r_i$ correspond to action $a_i$.
\end{proposition}
\begin{proof}
The correctness of the proposition is based on the fact that all the contexts in the dataset are i.i.d. The proof is similar to that of Proposition 1 in \citep{tennenholtz2020}. The difference is that their model is applied for confounding bandits in which only $M < d$ features of the context are observed while all features of the context are observed in our setting. Therefore we can apply their proof for the case $M = d$.
\end{proof}

\section{Proof of Theorem 2}
In this section, we provide the proof for Theorem 2.

\begin{proof}
Let the instantaneous regret  $r_t = \langle x_t, \theta^*_{\pi^*(x_t)} \rangle -\langle x_t, \theta^*_{a_t} \rangle$. We decompose $R_T$ as follows:
\begin{eqnarray*}
R_T & = & \sum_{t=1}^{T} \langle x_t, \theta^*_{\pi^*(x_t)} \rangle - \sum_{t=1}^{T}\langle x_t, \theta^*_{a_t} \rangle \\
& = & \sum_{t=1}^{T} (\langle x_t, \theta^*_{\pi^*(x_t)} \rangle -\langle x_t, \theta^*_{a_t} \rangle)\\
& = & \underbrace{\sum_{i \in S^{(1)}} r_i}_{\text{Term 1}} + \underbrace{\sum_{i \in S^{(2)}} r_i}_{\text{Term 2}},
\end{eqnarray*}
where $$S^{(1)} = \{1 \le t \le T| a_t \in \mathcal L(\mu) \},$$
$$S^{(2)} = \{1 \le t \le T| a_t \in \mathcal A \setminus \mathcal L(\mu) \}.$$
Since the sets $S^{(1)}, S^{(2)}$ are disjoint, we have that $|S^{(1)}| + |S^{(2)}| = T$. We let $|S^{(2)}| = T'$.  To bound $R_T$, we bound Term 1, Term 2 respectively.
\paragraph{Bounding Term 1}
\begin{eqnarray*}
\sum_{i \in S^{(1)}} r_i & = & \sum_{i \in S^{(1)}}\langle x_t, \theta^*_{\pi^*(x_t)} \rangle -\langle x_t, \theta^*_{a_t} \rangle \\
& \le & \sum_{i \in S^{(1)}} \langle x_t, \hat{\theta}_{a_t} \rangle -\langle x_t, \theta^*_{a_t} \rangle \\
& \le & \sum_{i \in S^{(1)}} \langle x_t, \hat{\theta}_{a_t} -  \theta^{\epsilon}_{a_t} \rangle \\
& \le & \sum_{i \in S^{(1)}} ||x_t|| ||\hat{\theta}_{a_t} -  \theta^{\epsilon}_{a_t}|| \\
& \le & \sum_{i \in S^{(1)}} \epsilon S_x \\
& \le &  \epsilon S_x (T -T')
\end{eqnarray*}
\paragraph{Bounding Term 2} Following the proof steps similar to that of mOFUL, we have that with probability at least $1 -\delta$, $\sum_{i \in S^{(2)}} r_i  \le \mathcal O(\sqrt{(K-L)T'})$.
\begin{eqnarray*}
\sum_{t \in S^{(2)}} r_i  & \le &  \sqrt{T_2 \sum_{t \in S^{(2)}} r^2_t} \\
& \le &  2\sqrt{T'\sum_{t \in S^{(2)}} \beta_{t-1}(\delta)\text{min}\{||x_t||_{\overline{V}^{-1}_{t-1, a_t}}, 1\}} \\
& \le &  2\sqrt{T'\beta_{T}(\delta)\sum_{a \in S^{(2)}} \sum_{i=1}^{N_T(a)} \text{min}\{||x_{I_i(a)}||_{\overline{V}^{-1}_{I_i(a), a}}, 1\}} \\
& \le &  2\sqrt{T'\beta_{T}(\delta)} \sqrt{d\sum_{a \in S^{(2)}} log(\lambda + \frac{N_T(a)S^2_x}{d})} \\
& \le &  2\sqrt{T'\beta_{T}(\delta)} \sqrt{d(K-L)log(\lambda + \frac{TS^2_x}{d(K-L)})} \\
& =  & \mathcal O(\sqrt{(K-L)T'}).
\end{eqnarray*}

Combining the bounds of Term 1 and Term 2, Theorem 2 holds.
\end{proof}

\section{Proof of Theorem 3}
In this section, we will derive a regret for the $\epsilon$-mOFUL-IPS algorithm. We first start with Lemma 1 which is stated as follows. This lemma provides a manner to upper bound the gap $R^+(\pi) - \hat{R}^M(\pi)$.
\begin{lemma} [Based on Theorem 1 of \citep{swaminathan15a}]
Given $(x_1, a_1), ..., (x_n, a_n)$ where $x_i$ is sampled uniformly at random in $\mathcal X$, $a_i \sim \mu(x_i)$, for $n \ge 16$ and for every $\pi \in \Pi^+$ we have, with probability at least $1 - \delta$,
$$R^+(\pi) \le  \hat{R}^M(\pi) + \mathcal O(\sqrt{Var_{\pi}log(N(1/n, \mathcal F_{\Pi}, 2n)/\delta)/n}),$$
\end{lemma}
\begin{proof}
Given a $\pi \in \Pi^+$, by definition we have
$$\hat{R}^M(\pi) = \frac{1}{|S|} \sum_{(x, a, r) \in S} r \text{min} \{\frac{\pi(a|x)}{\mu(a|x)}, M\}.$$
In the restricted space $\Pi^+$, the logging policy $\mu$ has a full support for all policy $\pi \in \Pi^+$. By the proof similar to the one of Theorem 1 of \citep{swaminathan15a}, we can obtain a generalization error bound for $R^+(\pi) - \hat{R}^M(\pi)$.
\end{proof}

We next provide the regret bound for Theorem 3. We start with the definition of the regret $R_T$:
$$R_T = \sum_{t=1}^{T} \langle x_t, \theta^*_{\pi^*(x_t)} \rangle - \sum_{t=1}^{T} \langle x_t, \theta^*_{a_t} \rangle,$$
where $a_t$ is the action chosen by the algorithm. Following the $\epsilon$-mOFUL-IPS algorithm, we define $a_t$ as follows:
    \begin{equation*}
  a_t=\begin{cases}
    \pi^+(x_t), & \text{if $x_t \in S$ and the condition at line 5 (Algorithm 3) is not satisfied},\\
    \text{argmax}_{a \in \mathcal A, \theta_a \in C_{t-1,a}} \langle x_t, \theta_a \rangle, & \text{if otherwise}.
  \end{cases}
\end{equation*}

To upper bound $R_T$, we will upper bound $r_t = \langle x_t, \theta^*_{\pi^*(x_t)} \rangle - \langle x_t, \theta^*_{a_t} \rangle$, for every $1 \le t \le T$. We consider the three cases:
\begin{itemize}
  \item Case 1: $a_t = \pi^+(x_t)$;
  \item Case 2: $a_t = \text{argmax}_{a \in \mathcal A, \theta_a \in C_{t-1,a}} \langle x_t, \theta_a \rangle$ and $a_t \in \mathcal A \setminus \mathcal L$;
  \item
  Case 3: $a_t = \text{argmax}_{a \in \mathcal A, \theta_a \in C_{t-1,a}} \langle x_t, \theta_a \rangle$ and $a_t \in \mathcal L$.
\end{itemize}

\textbf{Case 1:} $a_t = \pi^+(x_t)$. This holds when the condition at line 5 is not satisfied, i.e., $\langle x_t, \tilde{\theta}_{t, a_t} \rangle \le r^{\mu}_i\text{min}\{\frac{\pi^+(a^{\mu}_i|x^{\mu}_i)}{\mu(a^{\mu}_i|x^{\mu}_i)}, M\}$, where $(x^{\mu}_i, a^{\mu}_i,r^{\mu}_i) \in S$ such that $x^{\mu}_i = x_t$.

We consider the policy $\pi^*_{\mu}$ which is defined as $\pi^*_{\mu}(x) \in \text{argmax}_{a \in \mathcal U(x, \mu )} \langle x, \theta^*_{a} \rangle$. $\pi^*_{\mu}$ is regarded as the optimal policy for the unsupported actions. By this, we have that $ \langle x, \theta^*_{\pi^*(x)} \rangle \ge \langle x, \theta^*_{\pi^*_{\mu}(x)} \rangle$.

Now we consider the following two cases for the relation between $\pi^*(x_t)$ and $\pi^*_{\mu}(x_t)$:
\begin{itemize}
  \item if $\pi^*(x_t) = \pi^*_{\mu}(x_t)$ then
    \begin{eqnarray*}
    \langle x_t, \theta^*_{\pi^*(x_t)} \rangle - \langle x_t, \theta^*_{a_t} \rangle & = & \langle x_t, \theta^*_{\pi^*_{\mu}(x_t)} \rangle -  \langle x_t, \theta^*_{\pi^+(x_t)} \rangle\\
    & \le & \langle x_t, \tilde{\theta}_{t, a_t} \rangle - \langle x_t, \theta^*_{\pi^+(x_t)} \rangle \\
    & \le & r_i\text{min}\{\frac{\pi^+(a^{\mu}_i|x^{\mu}_i)}{\mu(a^{\mu}_i|x^{\mu}_i)}, M\} -  \langle x_t, \theta^*_{\pi^+(x_t)} \rangle,
    \end{eqnarray*}
  where the first inequality holds because  $a_t, \tilde{\theta}_{t, a_t} = \text{argmax}_{a \in S_t, \theta_a \in C_{t-1,a}} \langle x_t, \theta_a \rangle$. The second one holds because of the assumption of Case 1.
  \item if $\pi^*(x_t) \not = \pi^*_{\mu}(x_t)$ then $ \langle x, \theta^*_{\pi^*(x)} \rangle > \langle x, \theta^*_{\pi^*_{\mu}(x)} \rangle$. Therefore, $ \langle x, \theta^*_{\pi^*(x)} \rangle = \langle x, \theta^*_{\pi^+(x)} \rangle$, where policy $\pi^+$ is defined in Section 3.3 and it is regarded as the optimal policy for the supported actions by logging policy $\mu$. Thus,
      $$\langle x_t, \theta^*_{\pi^*(x_t)} \rangle - \langle x_t, \theta^*_{a_t} \rangle  =\langle x, \theta^*_{\pi^+(x)} \rangle - \langle x, \theta^*_{\pi^+(x)} \rangle = 0.$$
\end{itemize}

\textbf{Case 2:} $a_t = \text{argmax}_{a \in \mathcal A, \theta_a \in C_{t-1,a}} \langle x_t, \theta_a \rangle$ and $a_t \in \mathcal A \setminus \mathcal L$. We continue to consider the two cases: Case (i): $x_t \in S$, $\langle x_t, \tilde{\theta}_{t, a_t} \rangle > r^{\mu}_i\text{min}\{\frac{\pi^+(a^{\mu}_i|x^{\mu}_i)}{\mu(a^{\mu}_i|x^{\mu}_i)}, M\}$ and $a_t \in \mathcal U(x_t, \mu) \setminus \mathcal L$, and Case (ii) $x_t \not \in S$ and  $a_t \in \mathcal A \setminus \mathcal L$.

\begin{itemize}
  \item Case (i). For this case, we continue to consider two cases:
\begin{itemize}
  \item if $\pi^*(x_t) = \pi^*_{\mu}(x_t)$ then
   \begin{eqnarray*}
  \langle x_t, \theta^*_{\pi^*(x_t)} \rangle - \langle x_t, \theta^*_{a_t} \rangle & = & \langle x_t, \theta^*_{\pi^*_{\mu}(x_t)} \rangle - \langle x_t, \theta^*_{a_t} \rangle \\
  & \le & \langle x_t, \tilde{\theta}_{t, a_t} \rangle -  \langle x_t, \theta^*_{a_t} \rangle \\
  & = & \langle x_t, \tilde{\theta}_{t, a_t} - \theta^*_{a_t} \rangle \\
  & = & \langle x_t, \tilde{\theta}_{t, a_t} - \hat{\theta}_{t-1,a_t} \rangle +  \langle x_t, \hat{\theta}_{t-1,a_t} - \theta^*_{a_t}  \rangle \\
  & \le &  ||x_t||_{\overline{V}^{-1}_{t-1, a}} ||\tilde{\theta}_{t, a_t} - \hat{\theta}_{t-1,a_t}||_{\overline{V}^{-1}_{t-1, a}} + ||x_t||_{\overline{V}^{-1}_{t-1, a}} ||\hat{\theta}_{t-1,a_t} - \theta^*_{a_t}||_{\overline{V}^{-1}_{t-1, a}} \\
  & \le & 2\sqrt{\beta_{t-1}(\delta)} ||x_t||_{\overline{V}^{-1}_{t-1, a}}
    \end{eqnarray*}

  \item if $\pi^*(x_t) \not = \pi^*_{\mu}(x_t)$ then $ \langle x, \theta^*_{\pi^*(x)} \rangle > \langle x, \theta^*_{\pi^*_{\mu}(x)} \rangle$. Therefore, $ \langle x, \theta^*_{\pi^*(x)} \rangle = \langle x, \theta^*_{\pi^+(x)} \rangle$.
  \begin{eqnarray*}
  \langle x_t, \theta^*_{\pi^*(x_t)} \rangle - \langle x_t, \theta^*_{a_t} \rangle & = & \langle x_t, \theta^*_{\pi^+(x_t)} \rangle - \langle x_t, \theta^*_{a_t} \rangle \\
  & = &  \langle x_t, \theta^*_{\pi^+(x_t)} \rangle  - \langle x_t, \tilde{\theta}_{t, a_t} \rangle  +  \langle x_t, \tilde{\theta}_{t, a_t} \rangle - \langle x_t, \theta^*_{a_t} \rangle \\
  & \le &   \langle x_t, \theta^*_{\pi^+(x_t)} \rangle  - r^{\mu}_i\text{min}\{\frac{\pi^+(a^{\mu}_i|x^{\mu}_i)}{\mu(a^{\mu}_i|x^{\mu}_i)}, M\} + \langle x_t, \tilde{\theta}_{t, a_t} \rangle - \langle x_t, \theta^*_{a_t} \rangle \\
  & \le & \langle x_t, \theta^*_{\pi^+(x_t)} \rangle  - r^{\mu}_i\text{min}\{\frac{\pi^+(a^{\mu}_i|x^{\mu}_i)}{\mu(a^{\mu}_i|x^{\mu}_i)}, M\} + 2\sqrt{\beta_{t-1}(\delta)} ||x_t||_{\overline{V}^{-1}_{t-1, a}},
  \end{eqnarray*}
  where the first inequality, we use the assumption of Case 2 that $\langle x_t, \tilde{\theta}_{t, a_t} \rangle > r_i\text{min}\{\frac{\pi^+(a^{\mu}_i|x^{\mu}_i)}{\mu(a^{\mu}_i|x^{\mu}_i)}, M\}$. In the last inequality, we use the proof similar as above (when $\pi^*(x_t) = \pi^*_{\mu}(x_t)$ in Case 2).
\end{itemize}
  \item Case (ii). We have
  \begin{eqnarray*}
  \langle x_t, \theta^*_{\pi^*(x_t)} \rangle - \langle x_t, \theta^*_{a_t} \rangle  & = & \langle x_t, \theta^*_{\pi^*(x_t)} \rangle - \langle x_t, \theta^*_{a_t} \rangle \\
   & \le & 2\sqrt{\beta_{t-1}(\delta)} ||x_t||_{\overline{V}^{-1}_{t-1, a}},
  \end{eqnarray*}
  where in the last inequality, we use the proof similar as above. This case is as in the $\epsilon$-mOFUL without using an offline estimator.
\end{itemize}

\textbf{Case 3:} $a_t = \text{argmax}_{a \in \mathcal A, \theta_a \in C_{t-1,a}} \langle x_t, \theta_a \rangle$ and $a_t \in \mathcal L$. Similar to Case 2, we consider the two cases: Case (i): $x_t \in S$, $\langle x_t, \tilde{\theta}_{t, a_t} \rangle > r^{\mu}_i\text{min}\{\frac{\pi^+(a^{\mu}_i|x^{\mu}_i)}{\mu(a^{\mu}_i|x^{\mu}_i)}, M\}$ and $a_t \in  \mathcal L(\mu)$, and Case (ii) $x_t \not \in S$ and  $a_t \in \mathcal L$.

\begin{itemize}
  \item Case (i). Similar to Case 2, we continue to consider two cases:
  \begin{itemize}
    \item if $\pi^*(x_t) = \pi^*_{\mu}(x_t)$ then
        \begin{eqnarray*}
        \langle x_t, \theta^*_{\pi^*(x_t)} \rangle - \langle x_t, \theta^*_{a_t} \rangle & = & \langle x_t, \theta^*_{\pi^*_{\mu}(x_t)} \rangle -  \langle x_t, \theta^{*}_{a_t} \rangle \\
        & \le &  \langle x_t, \hat{\theta}_{a_t} \rangle -\langle x_t, \theta^*_{a_t} \rangle \\
        & \le &  \langle x_t, \hat{\theta}_{a_t} -  \theta^*_{a_t} \rangle \\
        & \le &  ||x_t|| ||\hat{\theta}_{a_t} -  \theta^{\epsilon}_{a_t}|| \\
        & \le &  \epsilon S_x
        \end{eqnarray*}

    \item if $\pi^*(x_t) \not = \pi^*_{\mu}(x_t)$ then $ \langle x, \theta^*_{\pi^*(x)} \rangle > \langle x, \theta^*_{\pi^*_{\mu}(x)} \rangle$. Therefore, $ \langle x, \theta^*_{\pi^*(x)} \rangle = \langle x, \theta^*_{\pi^+(x)} \rangle$.
        \begin{eqnarray*}
        \langle x_t, \theta^*_{\pi^*(x_t)} \rangle - \langle x_t, \theta^*_{a_t} \rangle & = & \langle x_t, \theta^*_{\pi^+(x_t)} \rangle - \langle x_t, \hat{\theta}_{a_t} \rangle \\
        & = &  \langle x_t, \theta^*_{\pi^+(x_t)} \rangle  - \langle x_t, \tilde{\theta}_{t, a_t} \rangle  +  \langle x_t, \tilde{\theta}_{t, a_t} \rangle - \langle x_t, \hat{\theta}_{a_t} \rangle \\
        & \le &   \langle x_t, \theta^*_{\pi^+(x_t)} \rangle  - r^{\mu}_i\text{min}\{\frac{\pi^+(a^{\mu}_i|x^{\mu}_i)}{\mu(a^{\mu}_i|x^{\mu}_i)}, M\} + \langle x_t, \tilde{\theta}_{t, a_t} \rangle - \langle x_t, \hat{\theta}_{a_t} \rangle \\
        & \le & \langle x_t, \theta^*_{\pi^+(x_t)} \rangle  - r^{\mu}_i\text{min}\{\frac{\pi^+(a^{\mu}_i|x^{\mu}_i)}{\mu(a^{\mu}_i|x^{\mu}_i)}, M\} + \epsilon S_x,
        \end{eqnarray*}
        where the first inequality, we use the assumption of Case 3 that $\langle x_t, \tilde{\theta}_{t, a_t} \rangle > r_i\text{min}\{\frac{\pi^+(a^{\mu}_i|x^{\mu}_i)}{\mu(a^{\mu}_i|x^{\mu}_i)}, M\}$. In the last inequality, we use the proof similar as above (when $\pi^*(x_t) = \pi^*_{\mu}(x_t)$ in Case 3).
    \end{itemize}
  \item Case (ii). We have
  \begin{eqnarray*}
  \langle x_t, \theta^*_{\pi^*(x_t)} \rangle - \langle x_t, \theta^*_{a_t} \rangle  & = & \langle x_t, \theta^*_{\pi^*(x_t)} \rangle - \langle x_t, \hat{\theta}_{a_t} \rangle \\
   & \le & \epsilon S_x,
  \end{eqnarray*}
  where in the last inequality, we use the proof similar as above.
\end{itemize}

For all cases, we can summarize as follows:
\begin{equation}
  \langle x_t, \theta^*_{\pi^*(x_t)} \rangle - \langle x_t, \theta^*_{a_t} \rangle = \begin{cases}
    r^{\mu}_i\text{min}\{\frac{\pi^+(a^{\mu}_i|x^{\mu}_i)}{\mu(a^{\mu}_i|x^{\mu}_i)}, M\} -  \langle x_t, \theta^*_{\pi^+(x_t)} \rangle, & \text{or} ,\\
    0, & \text{or}, \\
    2\sqrt{\beta_{t-1}(\delta)} ||x_t||_{\overline{V}^{-1}_{t-1, a}}, & \text{or}, \\
    \langle x_t, \theta^*_{\pi^+(x_t)} \rangle  - r^{\mu}_i\text{min}\{\frac{\pi^+(a^{\mu}_i|x^{\mu}_i)}{\mu(a^{\mu}_i|x^{\mu}_i)}, M\} + 2\sqrt{\beta_{t-1}(\delta)} ||x_t||_{\overline{V}^{-1}_{t-1, a}}, & \text{or}, \\
    \epsilon S_x, & \text{or}, \\
    \langle x_t, \theta^*_{\pi^+(x_t)} \rangle  - r_i\text{min}\{\frac{\pi^+(a^{\mu}_i|x^{\mu}_i)}{\mu(a^{\mu}_i|x^{\mu}_i)}, M\} + \epsilon S_x, & \text{or}.
  \end{cases}
\end{equation}
Thus, the regret of the $\epsilon$-mOFUL-IPS is bounded as
\begin{eqnarray*}
\sum_{t=1}^{T} (\langle x_t, \theta^*_{\pi^*(x_t)} \rangle - \sum_{t=1}^{T} \langle x_t, \theta^*_{a_t} \rangle) \le
\underbrace{\sum_{t \in B_1} (r^{\mu}_i\text{min}\{\frac{\pi^+(a^{\mu}_i|x^{\mu}_i)}{\mu(a^{\mu}_i|x^{\mu}_i)}, M\} -  \langle x_t, \theta^*_{\pi^+(x_t)} \rangle)}_{\text{Term 1}} +
\underbrace{\sum_{t \in B_2} \epsilon S_x}_{\text{Term 2}} + \\
+ \underbrace{\sum_{t \in B_3} 2\sqrt{\beta_{t-1}(\delta)} ||x_t||_{\overline{V}^{-1}_{t-1, a}}}_{\text{Term 3}} + \underbrace{\sum_{t \in B_4} (\langle x_t, \theta^*_{\pi^+(x_t)} \rangle  - r^{\mu}_i\text{min}\{\frac{\pi^+(a^{\mu}_i|x^{\mu}_i)}{\mu(a^{\mu}_i|x^{\mu}_i)}, M\})}_{\text{Term 4}},
\end{eqnarray*}
where the index sets $B_1$, $B_2$, $B_3$ and $B_4$ are defined as
$$B_1 = \{ 1 \le t \le T| \langle x_t, \theta^*_{\pi^*(x_t)} \rangle - \langle x_t, \theta^*_{a_t} \rangle \le  r^{\mu}_i\text{min}\{\frac{\pi^+(a^{\mu}_i|x^{\mu}_i)}{\mu(a^{\mu}_i|x^{\mu}_i)}, M\} -  \langle x_t, \theta^*_{\pi^+(x_t)} \rangle\},$$

$$B_2 = \{ 1 \le t \le T| \epsilon S_x \text{ exists in the upper bound of } \langle x_t, \theta^*_{\pi^*(x_t)} \rangle - \langle x_t, \theta^*_{a_t} \rangle \text{ in Eq (3)}\},$$

$$B_3 = \{ 1 \le t \le T| 2\sqrt{\beta_{t-1}(\delta)} ||x_t||_{\overline{V}^{-1}_{t-1, a}} \text{ exists in the upper bound of } \langle x_t, \theta^*_{\pi^*(x_t)} \rangle - \langle x_t, \theta^*_{a_t} \rangle \text{ in Eq (3)}\},$$

$$B_4 = \{ 1 \le t \le T| \langle x_t, \theta^*_{\pi^*(x_t)} \rangle - \langle x_t, \theta^*_{a_t} \rangle \le \langle x_t, \theta^*_{\pi^+(x_t)} \rangle  - r^{\mu}_i\text{min}\{\frac{\pi^+(a^{\mu}_i|x^{\mu}_i)}{\mu(a^{\mu}_i|x^{\mu}_i)}, M\} \}.$$

A constraint is $|B_1| + |B_2| + |B_3| + |B_4| = T$.

\textbf{Bounding Term 1}:  if $B_1  = \emptyset$ then Term 1 is zero. Assume that $B_1 \not \emptyset$. We apply Hoeffding's bound to the random variable $X_t = r^{\mu}_i\text{min}\{\frac{\pi^+(a^{\mu}_i|x^{\mu}_i)}{\mu(a^{\mu}_i|x^{\mu}_i)}, M\}$ for every $t \in B_1$. We note that $(x^{\mu}_i, a^{\mu}_i, r^{\mu}_i) \in S$ and $x_t = x^{\mu}_i$.  Set $\hat{V}_T = \frac{1}{|B_1|} \sum_{t \in B_1} r^{\mu}_i\text{min}\{\frac{\pi^+(a^{\mu}_i|x^{\mu}_i)}{\mu(a^{\mu}_i|x^{\mu}_i)}, M\}$. We have $|\hat{V}_T - \mathbb{E}[\hat{V}_T]| \le M\sqrt{\frac{ln(2/\delta)}{2|B_1|}}$. Therefore, $\hat{V}_T \le M\sqrt{\frac{ln(2/\delta)}{2|B_1|}} + \mathbb{E}[\hat{V}_T]$. Similar to the proof of Lemma 3.1 of \cite{Strehl10}, we have $\mathbb{E}[\hat{V}_T] \le R^+(\pi^+)$. Thus, we have
$$\sum_{t \in B_1} (r^{\mu}_i\text{min}\{\frac{\pi^+(a^{\mu}_i|x^{\mu}_i)}{\mu(a^{\mu}_i|x^{\mu}_i)}, M\} -  \langle x_t, \theta^*_{\pi^+(x_t)} \rangle) \le M\sqrt{\frac{ln(2/\delta)}{2|B_1|}} +  R^+(\pi^+) - \sum_{t \in B_1}\langle x_t, \theta^*_{\pi^+(x_t)} \rangle.$$

Further, we have $\mathbb{E}[\frac{1}{|B_1|} \sum_{i=1}^{|B_1|} \langle x_t, \theta^*_{\pi^+(x_t)} \rangle] = R^+(\pi^+)$.  Again, we apply Hoeffding's bound to $|B_1|$ random variables $\langle x_t, \theta^*_{\pi^+(x_t)} \rangle$.  We have $|\frac{1}{|B_1|} \sum_{i=1}^{|B_1|} \langle x_t, \theta^*_{\pi^+(x_t)} \rangle - R^+(\pi^+)| \le \sqrt{\frac{ln(2/\delta)}{|B_1|}}$ with probability at least $1 -\delta$.

Thus, $\sum_{t \in B_1} (r^{\mu}_i\text{min}\{\frac{\pi^+(a^{\mu}_i|x^{\mu}_i)}{\mu(a^{\mu}_i|x^{\mu}_i)}, M\} -  \langle x_t, \theta^*_{\pi^+(x_t)} \rangle) \le \mathcal O(\sqrt{\frac{ln(2/\delta)}{|B_1|}})$ with probability at least $1 -\delta$.

\textbf{Bounding Term 2}: $\sum_{t \in B_2} \epsilon S_x  = \epsilon S_x|B_2|$.

\textbf{Bounding Term 3}: By the proof similar as the one of Term 2 of Theorem 2, $\sum_{t \in B_3} 2\sqrt{\beta_{t-1}(\delta)} ||x_t||_{\overline{V}^{-1}_{t-1, a}} \le \mathcal O(\sqrt{(K-L)|B_3|})$ with probability $1 -\delta$.

\textbf{Bounding Term 4}: if $B_4  = \emptyset$ then Term 4 is zero. Assume that $B_4 \not = \emptyset$. We apply Lemma 1 for a set of $|B_4|$ samples $\{(x_{t_1}, a_{t_1}), ..., \}$, we achieve
$$\sum_{t \in B_4} (\langle x_t, \theta^*_{\pi^+(x_t)} \rangle  - r^{\mu}_i\text{min}\{\frac{\pi^+(a^{\mu}_i|x^{\mu}_i)}{\mu(a^{\mu}_i|x^{\mu}_i)}, M\}) \le \mathcal O(\sqrt{Var_{\pi}log(N(1/|B_4|, \mathcal F_{\Pi}, 2|B_4|)/\delta)/|B_4|}),$$
with probability $1 -\delta$.

Finally, with probability at least $1 -\delta$, we get that $\sum_{t=1}^{T} (\langle x_t, \theta^*_{\pi^*(x_t)} \rangle - \sum_{t=1}^{T} \langle x_t, \theta^*_{a_t} \rangle) \le \mathcal O(\epsilon S_x|B_2| + \sqrt{(K-L)|B_3|} +
\sqrt{Var_{\pi}log(N(1/|B_4|, \mathcal F_{\Pi}, 2|B_4|)/\delta)|B_4|})$.
\end{document}